\documentclass{article}
\usepackage[preprint,nonatbib]{neurips_2026}
\usepackage[utf8]{inputenc} 
\usepackage[T1]{fontenc}    
\usepackage{hyperref}       
\usepackage{url}            
\usepackage{booktabs}       
\usepackage{amsfonts}       
\usepackage{nicefrac}       
\usepackage{microtype}      
\usepackage{xcolor}         
\usepackage{amsmath}
\usepackage{amsthm}
\usepackage{amssymb}
\usepackage{graphicx}
\usepackage{float}
\usepackage{layout}
\theoremstyle{plain}
\newtheorem{theorem}{Theorem}
\usepackage{comment}
\usepackage{wrapfig}
\newcommand{\inv}{^{-1}}
\newcommand{\TI}{^{\mathrm{TI}}}
\usepackage{enumitem}

\title{A mathematical theory of balancing\\relational generalization and memorization}

\author{%
  Luke Cheng\\
  Center for Theoretical Neuroscience\\
  Columbia University\\
  New York, NY, 10027 \\
  \texttt{lc3616@columbia.edu}\And
  Samuel Lippl\\
  Center for Theoretical Neuroscience\\
  Columbia University\\
  New York, NY, 10027 \\
  \texttt{samuel.lippl@columbia.edu}
  }
\usepackage[backend=biber,sorting=none,style=numeric-comp]{biblatex}
\begin{document}

\definecolor{orange}{HTML}{ff8c00}
\definecolor{purple}{HTML}{9467bd}
\definecolor{green}{HTML}{1b7837}

\maketitle

\begin{abstract}
  Humans, animals, and modern machine learning models exhibit impressive abilities to learn complex behaviors and generalize these behaviors to unseen situations. This ability requires us to learn rules and regularities that allow for such generalizations. At the same time, in most complex environments, any rule will have its exceptions. How do learning systems balance between learning general regularities and memorizing exceptions? We argue that a lack of task paradigms has hindered the study of this essential ability. To address this gap, we introduce a novel task, transitive inference with exceptions, that tests for relational generalization and memorization of an exception to the relational rule. We then analytically characterize the behavior of a simple, theoretically tractable model of neural network learning (kernel ridge regression) across a broad family of representations and task parameters. We find that these models can balance between relational generalization and memorization, but unlike for transitive inference without an exception, successful generalization is sensitive to the specific representational geometry. We explain why this task is more challenging mechanistically by drawing on our analytical theory. Finally, we validate our theoretical insights in pretrained language models that are finetuned on ordered relations, finding that these models successfully generalize according to the transitive rule, but also make the kinds of systematic mistakes predicted by our theory. Overall, our theory shows how learning systems can balance between relational generalization and memorization, explains how this can go wrong, and emphasizes the need for new task paradigms designed to probe this ability.
\end{abstract}

\section{Introduction}
When learning to perform complex behaviors from limited data, humans, animals, and machine learning models must infer the rules and regularities of their environment \cite{kemp2008discovery,lake2017building} (Fig.~\ref{fig:setup}A). This allows them to generalize their experience to unseen situations involving familiar components, an ability also known as compositional generalization \cite{fodor1988connectionism,hupkes2020compositionality,keysers2020measuring}. For example, when learning a new language, we may notice that similar-sounding words (in related languages) often mean the same thing; this can help us learn new words faster. Understanding that certain environments share a stereotyped sequence of actions (e.g.\ at the airport we usually check our bag and then go through security) lets us better plan for the future \cite{baldassano2018representation}. And when learning chess, understanding that a queen is usually worth more to us than a knight will let us avoid making costly mistakes.

Yet in most complex environments, there will be exceptions to these rules (Fig.~\ref{fig:setup}B). For example, ``Gift'' in German means poison, not gift; we sometimes have to check a bag at the gate rather than before the security check; and in some cases sacrificing a queen can win us the game. Effective behavior requires learning these exceptions in addition to the regular structure (Fig.~\ref{fig:setup}C).

Thus, any learning system deployed in the real world will have to balance learning general rules with memorizing exceptions to those rules. While an extensive theoretical and empirical literature has investigated compositional generalization and memorization in isolation, it has remained unclear how learning systems can do both simultaneously \cite[though see][]{rumelhart1985learning}. Part of the issue is a lack of suitable task paradigms: while real-world tasks require a balance between rule learning and memorization, simple tasks amenable to theoretical study can often be solved by rules without considering any exceptions.

Here we address this gap, considering relational generalization as an important instance of compositional generalization \cite{halford2010relational,battaglia2018relational}. A canonical instance of a relational rule is given by transitivity: $A>B$ and $B>C$ implies $A>C$. Accordingly, testing for the ability to transitively generalize (``transitive inference,'' TI, Fig.~\ref{fig:setup}D) has been an important task paradigm for investigating relational reasoning in humans  \cite{piaget1928judgment,bryant1971transitive,ciranka2022asymmetric,nelli2023neural}, animals \cite{mcgonigle1977monkeys,davis1992transitive,grosenick2007fish,tibbetts2019transitive}, and (increasingly) neural networks \cite{de2001transitive,di2024geometrical,kay2024emergent,lippl2024mathematical,geerts2025relational,liu2025recoglab}. Conversely, both humans and animals (as well as neural networks) can also memorize intransitive relations (e.g.\ rock-paper-scissor-structures, ``transverse patterning,'' TP, Fig.~\ref{fig:setup}E) \cite{alvarado1992some,astur1998configural,dusek1998hippocampus}.

\begin{figure}
    \centering
    \includegraphics{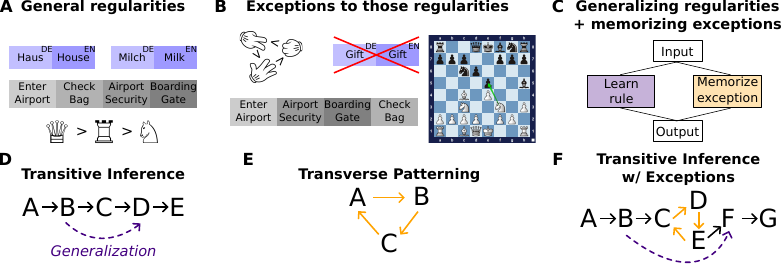}
    \caption{Complex behavior in the real world requires a mixture of general rule learning and memorizing exceptions. \textbf{A}, Identifying general regularities allows us to generalize to unseen situations. \textbf{B}, Most of these regularities have exceptions. (The chessboard shown here depicts an example of L{\'e}gal's Mate, a standard example where a player should sacrifice their queen, reprinted from \cite{chessfox_legals_mate}.) \textbf{C}, Optimal behavior in the real world requires learning these general regularities while also memorizing exceptions. \textbf{D}-\textbf{F}, Cognitive task paradigms probing the above abilities. \textit{D}, Transitive inference requires models or subjects to generalize using the transitive rule. \textit{E}, Transverse patterning tests whether subjects or models can learn intransitive relations. \textit{F}, We introduce a novel task: transitive inference with exceptions. This task requires subjects or models to both generalize according to the transitive rule and to memorize exceptions to the transitive rule.}
    \label{fig:setup}
\end{figure}

While transitive inference and transverse patterning have been extensively studied in isolation, it has remained unclear how learning systems perform on ordered relations that are mostly transitive but also include intransitive exceptions---a scenario we argue is likely common in complex decision-making tasks (see Section~\ref{sec:task-setup}). To study this, we make three major contributions:
\begin{enumerate}
    \item We introduce a novel task paradigm, transitive inference with exceptions (Fig.~\ref{fig:setup}F). This task is designed to probe relational generalization and memorization in a controlled setting.
    \item We analytically characterize a theoretical model of neural network learning (kernel ridge regression) across a general family of representations and task parameters. We find that kernel models can generalize successfully, but that this is sensitive to their representational geometry. We then leverage our analytical theory to mechanistically understand how these models generalize or fail to generalize.
    \item Finally, we validate our insights in pretrained language models finetuned on relational tasks.
\end{enumerate}
Overall, our findings demonstrate how simple learning systems can balance compositional generalization and memorization. At the same time, they highlight that compositional tasks that require such a balance are substantially more challenging than tasks that only require learning regularities. As a result, these tasks highlight novel failure modes and impose stronger constraints on which systems will be successful. We suggest that dealing with exceptions is a fundamental challenge for learning systems in the wild and future work should put a stronger emphasis on formulating analytically tractable task paradigms that capture this complexity.

\section{Related Work}
Compositional generalization has long been an important evaluation criterion for machine learning systems \cite{fodor1988connectionism,lake2017building,hupkes2020compositionality,lake2018generalization,johnson2017clevr,elmoznino2025towards,lepori2023break}. It has become increasingly central as modern learning systems have become more capable. While explicit constraints on the compositional operations in a neural network can guarantee compositional generalization \cite{schug2024discovering,wiedemer2023compositional,wiedemer2024provable,brady2025interaction,jarvis2023on}, our work emphasizes that neural networks often need to compositionally generalize in cases where explicit constraints don't work. Relational generalization, as a special instance of compositional generalization, focuses on generalizing to unseen situations by learning how different objects or items are related to each other \cite{halford2010relational,battaglia2018relational,whittington2020tolman,webb2024relational}. Transitive inference (TI; Fig.~\ref{fig:setup}D) is a classical task paradigm for assessing this ability in humans, animals, and neural networks by testing if subjects can generalize using the transitive rule \cite{vasconcelos2008transitive,jensen2017serial,geerts2025relational,liu2025recoglab}. In contrast, transverse patterning (TP; rock-paper-scissor-structure; Fig.~\ref{fig:setup}E) tests whether subjects are able to memorize fully intransitive relations. In the real world, many ordered relations may be largely transitive but also involve some exceptions to transitivity, a phenomenon that has been observed in a variety of contexts including competition between species, voting systems, and game theory \cite{trybula1961paradox,bar1988vicious,soliveres2018everything,klimenko2015intransitivity,conrey2016intransitive,arrow1950difficulty,austen1996information,bartholdi2025topology}. Our work builds on a long-standing literature that uses kernel models to understand neural network generalization \cite{jacot2018neural,chizat2019lazy,canatar2021out,canatar2021spectral,mallinar2022benign,zhou2024an,karkada2026predicting,abbe2024generalization,lippl2025when}. Prior work has suggested that such a perspective can help us understand finetuning in pretrained language models \cite{malladi2023kernel,afzal2026linearization}---this is consistent with our findings here. We refer the reader to a more detailed discussion of these related works in Appendix~\ref{app:related}. 
\section{Setup}
\label{sec:task}
\subsection{Task Setup}
\label{sec:task-setup}
To study relational inference, we consider a set of items $I_1,\dotsc,I_n$. To train models on a relation, we present pairs of items $(I_{j},I_{k})$ and an associated label $y_{j,k}\in\{-1,1\}$. If $y_{j,k}=1$, that indicates that $I_j>I_k$; if $y_{j,k}=-1$, that indicates that $I_j<I_k$. Items are represented by orthogonal vectors $x_1,\dotsc,x_n\in\mathbb{R}^d$ and each pair of items is represented by a simple concatenation of the two vectors, $x_{j,k}:=\begin{pmatrix}x_j&x_k\end{pmatrix}^T\in\mathbb{R}^{2d}$. This is a standard approach for studying compositional and relational inference in neural network models \cite[see e.g.][]{kay2024emergent,lippl2024mathematical,abbe2024generalization}. We note that all datasets we consider are antisymmetric: if they contain $(x_{j,k},1)$, they also contain $(x_{k,j},-1)$. If a model is trained on this data point, we say that it is trained on $I_j>I_k$.

\textbf{Transitive inference} tests whether subjects are able to infer an underlying ranking $I_1>\dotsb>I_n$. Models are trained on adjacent pairs $I_j>I_{j+1}$, and subsequently tested in their performance on $(I_j,I_k)$ for $|j-k|\geq2$. This tests whether models can use the transitive rule to generalize to novel combinations of items. \textbf{Transverse patterning} prevents models from inferring a consistent ranking system by adding the premise $I_n>I_1$, thus establishing an intransitive, cyclical relation. 

Our novel task, \textbf{transitive inference with exceptions}, introduces an arbitrary exception $I_p>I_q$ for $p>q$ (and $|p-q|> 1$). This means that the task consists of two ordered sections ($\mathcal{O}^{(1)}:=\{I_1,\dotsc,I_{q-1}\},\mathcal{O}^{(2)}:=\{I_{p+1},\dotsc,I_n\}$), with an intransitive section in the middle ($\mathcal{L}:=\{I_q,\dotsc,I_p\}$) (Fig.~\ref{fig:setup}f). We evaluate model predictions on four datasets (Fig.~\ref{fig:task}):%
\begin{wrapfigure}{r}{0.35\textwidth}
    \centering
    \includegraphics[width=\linewidth]{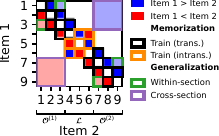}%
    \vspace{-1em}
    \caption{Transitive inference with exceptions. Illustration showing the relevant training and test sets as well as the associated relation. White squares reflect item pairs without an expected generalization.}%
    \vspace{-5em}
    \label{fig:task}
\end{wrapfigure}%
\begin{enumerate}[itemindent=0pt,leftmargin=*]
\item \textbf{Memorization of transitive pairs:} Does the model learn the training pairs where at least one item is in $\mathcal{O}^{(i)}$?
\item \textbf{\textcolor{orange}{Memorization of intransitive pairs:}} Does the model learn the training pairs where both items are in $\mathcal{L}$?
    \item \textbf{\textcolor{green}{Within-section generalization:}} Can the model generalize transitively within $\mathcal{O}^{(1)}$ or $\mathcal{O}^{(2)}$, i.e.\ on $\{(I_j,I_k)|I_j,I_k\in\mathcal{O}^{(i)};i=1,2\}$?
    \item \textbf{\textcolor{purple}{Cross-section generalization:}} Can the model generalize transitively across the sections, i.e.\ on $\{(I_j,I_k)|(I_j\in\mathcal{O}^{(1)}\wedge I_k\in\mathcal{O}^{(2)})\vee(I_j\in\mathcal{O}^{(2)}\wedge I_k\in\mathcal{O}^{(1)})\}$?
\end{enumerate}


For example, if subjects learn $A>B>C$, $E>F>G$, and $C>D>E>C$, an example of the desired cross-section generalization is $B>F$. This type of generalization allows the exception to create a local intransitive loop, without destroying the general transitive rule. We suggest that such a heuristic will yield successful relational generalization in the majority of cases. To ground our suggested heuristic in a concrete, mathematically precise example, we consider the one-on-one win probability of different hands in Texas Hold'em poker: $x\succ y$ iff the hand $x$ is more likely to win against hand $y$ (see Appendix~\ref{sec:poker}). This ordered relation contains intransitive loops \cite{bartholdi2025topology}. In analyzing this example, we found that typically, within-section and cross-section transitive generalization would have resulted in the appropriate generalization (Fig.~\ref{fig:poker}). 
While other intransitive relations may have a different structure, this provides an instance of a relation for which our suggested generalization is appropriate. 
\subsection{Model Setup}
\label{sec:model}
Given an input $x\in\mathbb{R}^{2d}$, we consider a random features model, $\phi(x)=\sigma(Wx)\in\mathbb{R}^{h}$, with random weights $W\in\mathbb{R}^{h\times 2d}$ and a nonlinearity $\sigma$. We denote the kernel induced by $\phi$ as $K(x,x'):=\phi(x)^T\phi(x')$.
In the limit of infinite width, $h\to\infty$, the kernel will take on three distinct values \cite{lippl2024mathematical}:
\begin{equation}
    K(x_{j,k},x_{j',k'})=\begin{cases}
        \kappa_s&\text{ if }j=j'\wedge k=k'\,\text{(identical trials)},\\
        \kappa_o&\text{ if }j=j'\vee k=k'\,\text{(overlapping trials)},\\
        \kappa_d&\text{ if }j\neq j'\wedge k\neq k'\,\text{(distinct trials)}.
    \end{cases}
\end{equation}
We call such kernels exchangeable, as the similarity between pairs of items doesn't depend on item identity, only on the overlap between the items. These kernels formalize the idea that we should learn about relations in terms of the interactions between different items rather than the specific feature similarity of those items. Given a dataset $\mathcal{D}=\{(x,y)\}$, we consider predictions made by a linear readout from $\phi(x)$, $f_w(x):=\phi(x)^Tw, w\in\mathbb{R}^h$ and consider ridge regression,
\begin{equation}
    \textstyle w^*:=\arg\min L(w),\quad L(w):=\sum_{x,y\in\mathcal{D}}(f_w(x)-y)^2+\tfrac{1}{c}\|w\|_2^2.
    \label{eq:opt}
\end{equation}
We consider $L_2$-regularization, as it often prevents overfitting to the training data and can therefore improve generalization at the cost of memorization \cite{nakkiran2021optimal,mei2022generalization}. Additionally, it is particularly relevant to our empirical experiments in Section~\ref{sec:dnns}, as we often want to limit the weight changes in finetuning language models, so as to prevent them from forgetting their pretraining data \cite{kirkpatrick2017overcoming,li2024revisiting,bethune2025scaling}.
Note that as $c\to\infty$, this network converges to the minimal $\ell_2$-norm solution $\min_w\|w\|_2^2\text{ s.t. }f_w(x)=y$.

We evaluate model performance in terms of the margin $yf(x)$. A positive margin indicates correct generalization and we consider a higher margin as indicative of better performance. More specifically, while we consider deterministic inputs that adhere to exchangeability exactly, in an empirical context (see e.g.~Section~\ref{sec:dnns}), these inputs might randomly vary. In that case, we assume that a higher margin will yield a more consistent output: intuitively, for points closer to the separating hyperplane, random variance in the input will be more likely to result in a perturbation that changes the prediction. For correct predictions (positive margins), that should result in increased accuracy; for incorrect predictions (negative margins), it should result in further decreases.

We will find that network behavior only depends on two factors:
\begin{align}
    \mbox{the ``conjunctivity factor'' }\alpha&:=1-2\frac{\kappa_o-\kappa_d}{\kappa_s-\kappa_d}\in[0,1],\\\mbox{the effective regularization parameter }\tilde{c}&:=\frac{c}{\kappa_s-\kappa_d}\in(0,\infty).
\end{align}
Intuitively, $\alpha$ captures how nonlinearly entangled the items are in $\phi(x)$. $\alpha=0$ (``compositional representation'') indicates that the two items are represented in orthogonal subspaces; $\alpha=1$ (``conjunctive representation'') indicates that they are fully entangled (Section~\ref{sec:extreme}).
The value of $\alpha$ depends on the nonlinearity, but typical values for a one-hidden-layer network are approximately in the range $(0,0.3)$ (e.g.\ a ReLU network has $\alpha\approx0.15$) \cite{lippl2024mathematical}.
\section{Theoretical Results}
\subsection{Main Theorem}
\textbf{Ranking systems.} A canonical approach to encoding an ordered relation is to assign each item a numerical rank and consistently choose the item with the higher rank. Ranking systems are inherently transitive and thus support transitive generalization, but cannot encode intransitive relations.

Intriguingly, we show below that kernel models learn an emergent ranking system, but combine it with an additional conjunctive response that can memorize exceptions. We characterize the model behavior across the entire family of representations, regularization strengths, and task parameters: 

\begin{theorem}
\label{thm}
    On TI with exceptions with task parameters $n,p,q$ and the training dataset $\mathcal{D}$, the predictions of a kernel model with an exchangeable representation are given by
    \begin{align}
        f(x_{j,k})&:=\begin{cases}
            r_j(\alpha,\tilde{c})-r_k(\alpha,\tilde{c})&\text{ for }x\notin\mathcal{D},\\
            m(\alpha,\tilde{c})y_{j,k}+(1-m(\alpha,\tilde{c}))(r_j(\alpha,\tilde{c})-r_k(\alpha,\tilde{c}))&\text{ for } x\in\mathcal{D},
        \end{cases}\\
        \mbox{where }m(\alpha,\tilde{c})&:=\tfrac{\alpha}{\alpha+\tilde{c}^{-1}},\quad r_j(\alpha,\tilde{c}):=r_j\TI(\alpha,\tilde{c})+r_j^{\mathrm{pert}}(\alpha,\tilde{c})
    \end{align}
    Here $r_j\TI$ is the ranking system emergent for standard TI \cite{lippl2024mathematical} \eqref{eq:rti} and $r_j^{\mathrm{pert}}$ (defined in \eqref{eq:rpert}) is the perturbation to the ranking system arising from the exception $I_p>I_q$ (Fig.~\ref{fig:thm}A). For the training data, the term $m(\alpha,\tilde{c})y_{j,k}$ demonstrates the conjunctive response that can memorize the exception.
\end{theorem}

\textbf{Proof sketch.}
We analytically solve the ridge regression problem (\ref{eq:opt}) through the dual form, which involves inverting the training set kernel $K$. We relate this to the inverse for the original TI problem \cite{lippl2024mathematical}, which allows us to express the analytical solution in terms of a perturbation on top of that solution. We then simplify that solution algebraically. For the detailed derivation, see Appendix~\ref{app:proof}.

\textbf{An encoding model.} To provide an intuition for the behavior of our model, we show in Appendix~\ref{app:intuition} that any exchangeable representation is equivalent, via an orthonormal change of basis, to a four-hot representation concatenating two vectors representing the first and the second item (the ``compositional population''), a vector representing the specific conjunction of the two items (the ``conjunctive population''), and a unit that is constantly active (the ``bias'') (Fig.~\ref{fig:thm}C). Different representational geometries correspond to different weightings of these populations. We can therefore rewrite model behavior in terms of separate readouts from these distinct populations: $f(x_{j,k})=r_j-r_k+t_{j,k}+b$. The readout from the bias term is zero, the readout from the compositional population implements the ranking system, and the readout from the conjunctive population (which we will call the ``conjunctive readout'') implements an additional conjunctive response that is only active on the training set. We note that the conjunctive readout allows the model to memorize the intransitive loop.

\begin{figure}
    \centering
    \includegraphics{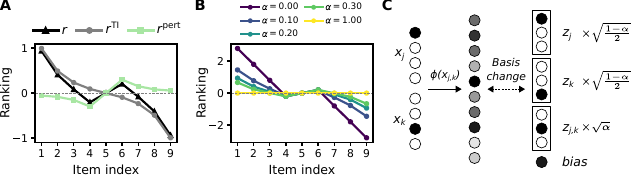}
    \caption{Illustration of the theorem. \textbf{A}, The ranking system consists of the original rank from TI, $r\TI$, plus a perturbation $r^{\mathrm{pert}}$ (example here uses $\alpha=0.2$). \textbf{B}, Example ranking systems (for $n=9,p=6,q=4,\tilde{c}\to\infty$). \textbf{C}, Any exchangeable representation is equivalent, via an orthonormal change of basis, to a four-hot representation where two units represent each input item $x_j$ and $x_k$ separately, one unit represents their conjunction, and one unit is active for all inputs. Different representational geometries are reflected in different weightings of these populations.}
    \label{fig:thm}
\end{figure}

\subsection{The Extreme Cases: Fully Compositional or Fully Conjunctive Representations}
\label{sec:extreme}






To get an intuition for model behavior, we first consider a fully additive model ($\alpha = 0$). This corresponds to the case where only the compositional population is active and the model implements a ranking system on both training and test set (as noted in \cite{lippl2024mathematical}): it infers a rank $r_j$ for each item $I_j$ and computes $f(x_{j,k})=r_j-r_k$.
This is an inherently transitive computation as $f(x_{j,k})>0$ (corresponding to the judgement that $I_j>I_k$) is equivalent to $r_j>r_k$. While the model will therefore learn the transitive pairs, it will necessarily make errors on the intransitive pairs.

More specifically, the model will learn a flat ranking system with respect to the intransitive items, while learning the correct ranking system with respect to all other items (Fig.~\ref{fig:thm}B, purple line). The flat ranking system allows the model to distribute its errors evenly across all intransitive pairs---thus minimizing the mean squared error. This means that an additive representation can generalize both within and across sections, but cannot memorize the exceptions.

In the opposite extreme ($\alpha=1$), only the conjunctive population is active. Such a model can learn arbitrary training data, as it effectively implements a lookup table, $f(x_{j,k})= t_{j,k}$. But it also has no understanding of how different items are related to each other and is therefore unable to generalize.

Overall, the extreme representations ($\alpha=0$ or $\alpha=1$) represent the extreme strategies described in the introduction: either fully rule-based generalization or full memorization. Our insights here are consistent with prior insights on TI and TP. We now turn to the more interesting case of $\alpha\in(0,1)$, which we will see naturally balances between memorization and generalization.

\subsection{Task Performance Depends on Representational Geometry and Regularization}
We use our analytical solution, which is applicable across representations and task parameters, to develop a detailed understanding of training and test performance. 

\begin{figure}
    \centering
    \includegraphics[width=0.7\textwidth]{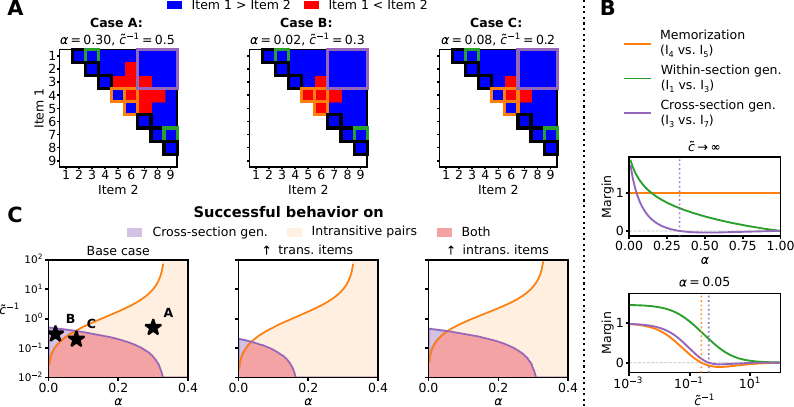}
    \caption{Behavior of the kernel model across task space. \textbf{A}, Generalization behavior for $n=9,p=6,q=4$. We show three distinct parameter settings that lead to: successful memorization, but errors in generalization (case A); successful generalization, but errors in memorization (case B); and successful memorization and generalization (case C). \textbf{B}, Changes in the margin as a function of $\alpha$ or $\tilde{c}$. Values where the margin becomes negative (indicating systematic errors) are highlighted with dashed vertical lines. \textbf{C}, Phase diagrams which show the effect of changing $\alpha$ and $\tilde{c}^{-1}$ on performance.}
    \label{fig:phase}
\end{figure}

\textbf{Universal properties of the ranking system.} We find that the ranking system has a stereotyped shape: it is always monotonically decreasing on the two transitive lists and typically monotonically \emph{increasing} over the intransitive loop (Fig.~\ref{fig:thm}B). We emphasize that this non-monotonicity in the ranking system only arises for $\alpha\in(0,1)$; for both extreme cases, the ranking system is flat on the intransitive loop. This demonstrates that a model balancing between relational generalization and memorization will exhibit new and interesting behaviors that go beyond a simple mixture of a model that can only generalize and a model that can only memorize.

\textbf{Generalization.} Because the ranking system is always monotonically decreasing within the two transitive lists (Appendix~\ref{app:ranking-system}), we can immediately conclude that any $\alpha>0$ will yield \textbf{successful within-section generalization}. For a suitable representational geometry, the model is also capable of \textbf{cross-section generalization} (e.g.\ Fig.~\ref{fig:phase}A, middle and right). However, this is not guaranteed (e.g.\ Fig.~\ref{fig:phase}A, left). 
Specifically, the ranking system's non-monotonocity in the intransitive section propagates into the transitive sections as well: the pre-exception items' ranks are biased downwards and the post-exception items' ranks are biased upwards. For a sufficiently large $\alpha$ or a sufficiently small $\tilde{c}$, this means that the model makes systematic errors in generalization, predicting, in particular, that $I_{q-1}<I_{p+1}$ (Fig.~\ref{fig:phase}B). We emphasize that kernel models never make such mistakes for standard TI (for $\alpha<1$); this highlights that compositional tasks involving exceptions are more challenging than tasks that only require learning a relational structure (such as TI). Accordingly, more compositional representations (with small $\alpha$) may actually be more important for such tasks.


\textbf{Memorization.} For $\alpha\in(0,1)$ and $\tilde{c}\to\infty$, the model will memorize all training data points. However, as $\tilde{c}$ decreases, the model eventually starts making systematic errors on the exception pairs (for all task variants except TP): while it still gets $(I_p,I_q)$ right, it starts getting $(I_q,I_{q+1})$ and $(I_{p-1},I_p)$ wrong (Fig.~\ref{fig:phase}B, bottom). Our analytical solution shows that this is because as $\tilde{c}$ decreases, $m(\alpha,\tilde{c})$ becomes smaller, meaning that the ranking system more strongly influences training set performance. The errors therefore arise from the non-monotonicity in the ranking system. Notably, for smaller $\alpha$, the model makes errors for a broader range of $\tilde{c}$ values. This demonstrates that more conjunctive representations will more easily account for an intransitive structure. We emphasize that on transverse patterning, the model never makes any systematic errors, regardless of $\tilde{c}$. This demonstrates that task structures that partially exhibit relational structure (such as TI with exceptions) can make memorization of exceptions more challenging.

Overall, our analysis demonstrates two things: On the one hand, even a relatively simple learning model can balance relational generalization with memorization of exceptions. On the other hand, this is much more challenging than \emph{only} learning a relation (as in TI), or only having to memorize (as in TP). We summarize these insights in terms of a single diagram demonstrating the regions of the $(\alpha,\tilde{c})$ where the model successfully generalizes and memorizes, and the regions where it fails to do one of the two or both (Fig.~\ref{fig:phase}C, left). A notable and surprising takeaway from our analysis is that $L_2$-regularization (which aids in-distribution generalization) is actually counterproductive for out-of-distribution generalization and can induce systematic errors. Furthermore, our analysis highlights that TI with exceptions gives rise to a trade-off between more conjunctive representations (better for memorization) and more compositional representations (better for generalization).

\textbf{Effects of task parameters.}
We now investigate the impact of changing different task parameters. We focus on cases where the intransitive loop is still located in the middle of the list (i.e.\ $q+1=n-p$), and vary either the length of the two transitive sections (``$\uparrow$ transitive items'') or the length of the intransitive loop (``$\uparrow$ intransitive items'') (Fig.~\ref{fig:phase}C, middle and right). While memorization ability is not strongly affected by either change, additional transitive items renders cross-section generalization much more challenging, as indicated by the area of successful generalization being much smaller. Thus, increasing the length of the transitive sections makes generalization much more challenging, whereas increasing the length of the intransitive loop has a much weaker effect.

\subsection{A Mechanistic Explanation for the Non-Monotonic Ranking System}
\label{sec:intuition}
The central challenge to both cross-section generalization and memorization of intransitive pairs is the non-monotonicity in the ranking system. We now leverage our analytical solution to mechanistically explain this non-monotonicity. We base our intuition around the encoding model we provided above, which expressed model behavior in terms of distinct latent populations responding to single components and component conjunctions: $f(x_{j,k})=r_j-r_k+t_{j,k}$. We express the effect of $L_2$-regularization directly in terms of these components, $C(r,t):=\tfrac{2}{1-\alpha}\|r\|_2^2+\tfrac{1}{\alpha}\|t\|_2^2$ (see Appendix~\ref{app:intuition}).
Intuitively, more additive representations will incur a higher cost for using $t_{j,k}$ and more conjunctive representations will incur a higher cost for using $r_j$.
For simplicity, we focus on $\tilde{c}\to\infty$. Thus, the model needs to match the prescribed labels exactly while minimizing $C(r,t)$.

To gain an intuition for how $C(r,t)$ shapes model behavior, we assume that the conjunctive population is only used where it is necessary, i.e.\ on the intransitive pairs, since a ranking system is transitive and cannot encode them\footnote{In practice, the conjunctive population will also partially take over predictions on the transitive pairs and therefore shrink the ranking system on the transitive section. We here neglect this aspect, to compactly explain the emergence of a non-monotonic ranking system.}. This means that the ranking system necessarily encodes the transitive pairs and thus decreases by one for each item in the transitive sections: $r_{i+1}=r_i-1$. $C(r,t)$ consists of a mixture of $\|r\|_2^2$ and $\|t\|_2^2$. To understand how each part shapes model behavior, we consider the solution that minimizes either just $\|r\|_2^2$ or just $\|t\|_2^2$. In the first case (minimizing $\|r\|_2^2$), each section's emergent ranking system will be centered around zero (Fig.~\ref{fig:intuition}A, solid blue line). This means that the ranking system will be increasing in the intransitive section. This results in systematically incorrect predictions on the intransitive pairs, which the conjunctive population $t_{j,k}$ needs to correct, incurring a high cost $\|t\|_2^2$ (Fig.~\ref{fig:intuition}B, right). On the other hand, if we minimize $\|t\|_2^2$, the ranking system will be flat for all items participating in the intransitive pairs. $t_{j,k}$ still needs to encode the intransitive pairs but incurs a lower cost than for other ranking systems (Fig.~\ref{fig:intuition}B, left). As a result, the transitive sections' ranking systems will be shifted away from zero (Fig.~\ref{fig:intuition}A, dashed blue line), incurring a higher cost $\|r\|_2^2$.

Thus, minimizing $\|t\|_2^2$ yields a flat ranking system, whereas minimizing $\|r\|_2^2$ yields a non-monotonic ranking system. Minimizing a mixture of both costs, $C(r,t)$ therefore finds a compromise between these two systems, resulting in a non-monotonic ranking system, albeit not one as extreme as the parallel ranking systems.

How does $\alpha$ change this tradeoff? For small $\alpha$, memorization is extremely costly whereas the ranking system is comparatively cheap. Thus, the model will focus on minimizing the memorization cost and keep the ranks of the intransitive loop relatively flat. In contrast, for large $\alpha$, memorization is cheap and the ranking system costly, yielding a more non-monotonic ranking system.

While we focused on $\tilde{c} \to \infty$, this intuition extends to smaller $\tilde{c}$ as well. In that case, the model is no longer constrained to exactly encoding the input-output data and rather minimizes a weighted sum of the mean squared error (MSE) and the weight norm $\|w\|_2^2=C(r,t)$ (see \eqref{eq:opt}).  Thus any errors on the intransitive pairs made by the ranking system can be compensated not only by increasing the conjunctive coefficient, but also by allowing an increased MSE (in order to decrease the weight norm $C(r,t)$). This similarly increases the non-monotonicity in the ranking system.

\section{Finetuning Behavior in Pretrained Language Models}
\label{sec:dnns}
\begin{figure}
    \centering
    \includegraphics[width=0.8\linewidth]{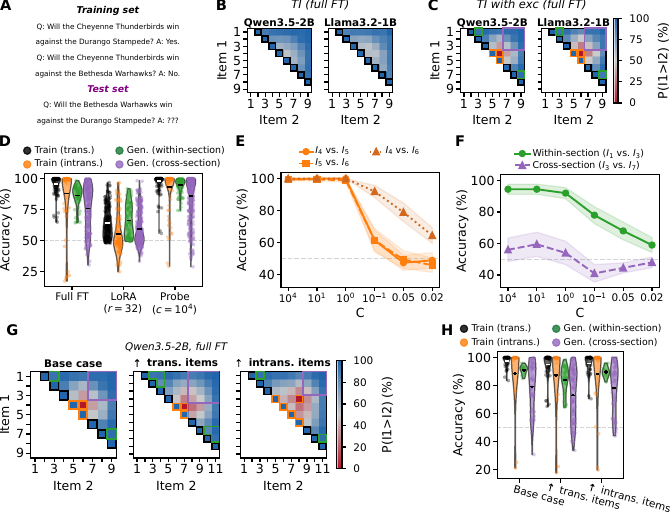}
    \caption{Finetuning pretrained language models (PLMs) on relational data with exceptions. \textbf{A}, Example of the data the models are finetuned on. \textbf{B},\textbf{C}, Generalization behavior of different PLMs on \textit{B}, TI and \textit{C}, TI with exceptions ($n=9, p=6, q=4$). Here we plot performance on fake sports team, other modalities see Figs.~\ref{fig:ti-acc-full}, \ref{fig:exc-acc-full}. For \textit{C}, see Fig.~\ref{fig:task} for the desired generalization. \textbf{D}, Accuracy on distinct subsets of the data across the different finetuning methods. Each dot corresponds to a specific rank comparison, task modality, and model. We summarize the data in terms of the mean and 95\% binomial confidence intervals (often too small to see). \textbf{E}, \textbf{F}, Performance of linear probing in Qwen3.5-4B (fake teams modality) on \textit{E}, intransitive pairs and \textit{F}, generalization pairs. Shaded regions represent 95\% binomial confidence intervals. \textbf{G}, \textbf{H}, Generalization behavior on full finetuning for the base task, a task with additional transitive items ($n=11,p=7,q=5$), or a task with additional intransitive items ($n=11,p=8,q=4$). Heatmaps in \textit{G} show fake teams (see Fig.~\ref{fig:exc-mid}, \ref{fig:exc-far} for other cases), violin plots in \textit{H} pool data across models and modalities.}
    \label{fig:dnns} 
\end{figure}
We now evaluate our theoretical predictions in pretrained language models (PLMs), which we finetune on novel hierarchies. We present these hierarchies in two different formats: either the likely winner between two sports teams (which are sampled from either real team names or fake team names; $n=100$ random seeds per setting), or the likely winner between two hands in Texas Hold'em Poker (using either the ground truth or the exact reverse; $n=50$ random seeds per setting). We present each relation as a question/answer prompt (see Fig.~\ref{fig:dnns}A). We consider three distinct finetuning strategies: 1) full finetuning (full FT), 2) low-rank adaptation (LoRA, \cite{hu2022lora}), and 3) finetuning of the linear readout head with $L_2$-regularization with respect to the initial values (linear probe). We consider four different language models (Llama3.2-1B/3B \cite{touvron2023llama} and Qwen3.5-2B/4B \cite{qwen3.5}; for full finetuning, we only consider the two smaller models). Finally, we evaluate four different relational tasks (TI and three variants of TI with exceptions). This allows us to assess PLMs' relational inference abilities across a broad range of different factors. We summarize our findings here and describe the setup and results in more detail in Appendix~\ref{app:dnns}.

As noted in Section~\ref{sec:model}, we connect the accuracy metrics of the PLMs to the predictions made by the kernel model in two ways: 1) The kernel model making a correct/incorrect prediction is analogous to the PLM behavior having above/below chance accuracy. 2) The overall margin may predict the overall magnitude of the accuracy.

We emphasize that our relational learning task is a practically relevant scenario: Language models are getting exposed to a lot of relational information in their training data. Extending these relations to novel combinations of items would let them generalize to a broad range of new inputs \cite{battaglia2018relational}.

We first evaluated the models on a transitive inference task without any exceptions. We found that the models generalized successfully on this task (Fig.~\ref{fig:dnns}B; Appendix~\ref{app:ti}). Moreover, their generalization behavior shared important characteristics with that of humans and animals (e.g.\ lower accuracy for intermediate item comparisons, see Appendix~\ref{app:behavioral-effects}). Finally, PLMs generalized better on tasks that were consistent with their prior knowledge (i.e.\ poker) than tasks that were inconsistent (i.e.\ poker with a reversed hierarchy), a phenomenon that has previously been observed in both PLMs and humans \cite{evans1983conflict,evans1995belief,dasgupta2022language,lampinen2024language} (see Appendix~\ref{sec:content-effects}). Overall, this demonstrates that language models are able to compositionally generalize over relational data that they were finetuned on and contributes to a growing literature evaluating Transformers and PLMs on transitive inference \cite{wu2025transitive,geerts2025relational,liu2025recoglab}. We note that linear probing and full finetuning generally led to higher performance than LoRA (Fig.~\ref{fig:ti-summary}) \cite{biderman2024lora}.

Next, we considered models finetuned on TI with exceptions. Consistent with our theory, we found that models broadly generalized above chance within-section (Fig.~\ref{fig:dnns}C,D; Appendix~\ref{app:exc}), whereas memorization of the intransitive pairs and cross-section generalization was much more inconsistent, with some models having high performance and others making systematic errors (indicated by accuracies below 50\%). Thus, finetuned PLMs can balance relational generalization and memorization, but some aspects of generalization (pinpointed by our theory) are brittle.

We then considered the impact of regularization. We first considered training set behavior, focusing on the intransitive loop (Figs.~\ref{fig:dnns}E, \ref{fig:exc-sweep}). We found that stronger regularization generally decreased accuracy, with substantially faster decreases on $x_{4,5}$ and $x_{5,6}$ than $x_{4,6}$, as predicted by our theory. We then considered generalization behavior, comparing the impact of regularization on within-section and cross-section generalization (exemplified by $x_{1,3}$ vs.\ $x_{3,7}$, the most challenging generalization pair according to our theory, see Fig.~\ref{fig:phase}B). We again observed the behavior predicted by our theory: as regularization becomes stronger, predictions on $x_{3,7}$ eventually become systematically incorrect (Figs.~\ref{fig:dnns}F, \ref{fig:gen-sweep}). In contrast, within-section accuracy decreases as regularization increases (consistent with the kernel model's output becoming smaller), but never goes below chance.

Finally, we varied the task parameters, either increasing the number of generalizable items ($n=11,p=7,q=5$) or increasing the length of the intransitive loop ($n=11,p=8,q=4$). Consistent with our theory, we found that increasing the number of generalizable items yielded a stronger decrease in performance than increasing the number of intransitive items (Fig.~\ref{fig:dnns}G,H, Fig.~\ref{fig:task-sweep}). Overall, these results indicate that our theory captures key aspects of how pretrained language models balance relational generalization and memorization.
\section{Discussion}
Complex behavior in the real world requires not just inferring general rules, but also memorizing exceptions to those rules. To study how learning systems balance these opposing demands, we have introduced a novel task paradigm grounded in the relational learning literature---transitive inference with exceptions. We analytically characterized the behavior of kernel models on this task, finding that a successful balance between relational generalization and memorization was sensitive to the models' representational geometry, and subsequently validated our theory in pretrained language models finetuned on novel relations.

Our theory highlights (and our experiments confirm) that exceptions render successful generalization much more challenging. This highlights that task paradigms like TI with exceptions can 1) help us understand the conditions under which learning systems generalize on such complex tasks and 2) can serve as a benchmark for the development of learning systems that avoid their pitfalls. Better learning systems could, for example, separate the learning of general structure and exceptions into distinct streams of processing \cite{schapiro2017complementary} or use meta-learning to imbue the appropriate inductive biases \cite{lake2023human}.

We focused on kernel models with exchangeable representations. Kernel models are an influential model in deep learning theory, but simplify important aspects of neural network learning. Future work should investigate how feature-learning neural networks generalize on TI with exceptions.
Moreover, exchangeable representations capture the idea that we should be able to relationally generalize over different items without any particular knowledge about the items themselves; by relaxing this assumption, we could investigate how prior knowledge (e.g.\ about poker) could affect relational reasoning \cite{dasgupta2022language,domine2023exact,lampinen2024language}. Finally, we have considered one specific instance of compositional generalization. Expanding this to a broader battery of tasks could highlight important distinctions and convergences in balancing generalization and memorization across different compositional structures. Our conceptual and theoretical framework provides a foundation for such future advances. Overall, we suggest that explicitly modeling the effect of exceptions to general rules is critical for understanding and improving compositional generalization in the real world.

\section*{Acknowledgments}
We thank Larry Abbott, Ethan Hwang, Chris Iyer, Berkan Ottlik, Erica Shook, and Marcus Triplett for helpful discussions and comments. We thank Justin Buck for help in brainstorming intransitivities in games and Kollin Wasserlein for generating fictional sports team names. This work was supported by a grant from the Simons Foundation International 542939SPI, LFA; SFI-AN-NC-GB-Culmination-00003215-01.
\printbibliography

\newpage
\appendix
\renewcommand{\thefigure}{S\arabic{figure}} 
\renewcommand{\theequation}{S\arabic{equation}}
\section{Extended Related Work}
\label{app:related}
\textbf{Compositional generalization.}
As modern machine learning models have become increasingly capable, generalizing to unseen situations has become an increasingly important challenge \cite{abbe2024generalization,press2023measuring,berglund2024the}. Compositional or systematic generalization tests whether models are able to generalize to novel combinations of familiar components \cite{hupkes2020compositionality}. In deep neural networks, explicit, pre-specified constraints on the compositional operation can guarantee compositional generalization \cite{jarvis2023on,wiedemer2023compositional,wiedemer2024provable,brady2025interaction,schug2024discovering}; alternatively, these compositional principles can be meta-learned \cite{wu2023adaptive,lake2023human,redhardt2025scaling} (though see \cite{pmlr-v140-mitchell21a}). These studies usually consider tasks that follow an exact rule; we introduce a task paradigm that highlights the need for compositional generalization even in cases where explicit constraints would not work.

\textbf{Kernel models.} As described in Section~\ref{sec:model}, linear readout models trained with $L_2$-regularization (``ridge regression'') can be understood in terms of the kernel induced by their representation. Ridge regression is both a common statistical learning method itself and captures the behavior of other popular learning algorithms. In particular, linear readout learning with gradient descent (from zero initialization) converges to the same solution \cite{soudry2018implicit,gunasekar2018characterizing,ji2020gradient}. Moreover, an influential line of work in deep learning theory approximates neural networks by their first-order Taylor expansion (the ``neural tangent kernel''). For large initial weights or wide neural networks, this approximation remains accurate throughout training, meaning that kernel models capture neural network learning exactly in this regime \cite{jacot2018neural,chizat2019lazy}. As a result, kernel models have been a highly influential framework for theoretically understanding neural network learning \cite{canatar2021out,canatar2021spectral,mallinar2022benign,zhou2024an,karkada2026predicting}, including in compositional generalization settings \cite{abbe2024generalization,lippl2025when}. In finetuning, in particular, weight changes are often very small, suggesting that kernel models might be a particularly relevant framework \cite{malladi2023kernel,afzal2026linearization}. Our findings support this, as our theory accurately captures finetuning behavior in pretrained models. In other contexts, inductive biases that are not captured by kernel models (and that our theory thus doesn't speak to) play an important role for generalization \cite{chizat2020implicit,woodworth2020kernel,vyas2023empirical,boix2023can,tong2025learning}, including in finetuning \cite{lampinen2018an,kumar2022fine,lippl2024inductive,anguita2026theory}.

\textbf{Transitive inference and transverse patterning.} Transitive inference (TI) paradigms assess whether subjects are able to use the transitive rule to generalize learned relations. This is a canonical task for studying relational generalization in humans and animals (from monkeys and rats to wasps and fish) \cite{bryant1971transitive,mcgonigle1977monkeys,davis1992transitive,grosenick2007fish,tibbetts2019transitive} (see also related task paradigms, e.g.\ \cite{zeithamova2010flexible,daw2011model,momennejad2017successor,yang2024large,tong2026boule}). Prior work has modeled these choice behaviors through models learning a ranking system \cite{jensen2015implicit,jensen2019discovering,ciranka2022asymmetric}, a strategy that intriguingly also emerges in kernel models trained on this task \cite{nelli2023neural,lippl2024mathematical}. Notably, animals' predisposition to ranking different items is so strong that they arrive at this strategy even in cases where trial-by-trial feedback is completely random (``superstitious learning'') \cite{jin2022superstitious}. Transverse patterning is a common task for studying whether animals can learn nonlinear relations \cite{alvarado1992some,astur1998configural,dusek1998hippocampus}. \cite{graham2025asymmetric} investigate the effect of a later change of the overall ranking system and whether humans are able to update their ranking system. Their setup still assumes a global ranking system without any exception, just one that changes over time. Investigating the interactions between intransitive loops (as in our setup) and changing relational orders (as in their setup) would be an interest subject of future work.

TI has been used to assess relational generalization in neural networks \cite{de2001transitive,nelli2023neural,di2024geometrical,lippl2024inductive,kay2024emergent}, including in-context learning in transformers and pretrained language models (PLMs) \cite{geerts2025relational,liu2025recoglab,wu2025transitive}. We expand on this work by studying relational generalization as a result of in-weight learning in PLMs.

\textbf{Intransitive loops in relational orders.}
Intransitive loops commonly arise when a multi-dimensional comparison structure is projected into a low-dimensional space \cite{trybula1961paradox,bar1988vicious,klimenko2015intransitivity}. A canonical example is inter-species competition, where complementary strengths and vulnerability might mean that A out-competes B, B out-competes C, and C out-competes A \cite{may1975nonlinear,gilpin1975limit,buss1980competitive,sinervo1996rock,huisman1999biodiversity,kerr2002local,soliveres2018everything}.
This has been observed in many different natural environments \cite{buss1980competitive,huisman1999biodiversity,sinervo1996rock,soliveres2018everything} and may support co-existence of different species within the same niche \cite{laird2006competitive}.
Further, ranked-choice voting system potentially give rise to intransitive preference orderings \cite{arrow1950difficulty}. In game theory, these concepts have been explored with ``intransitive dice'' \cite{Savage01051994,conrey2016intransitive}; along these lines, we expand on a prior analysis of choice intransitivity in poker \cite{bartholdi2025topology}. Overall, this highlights that potential intransitivity is ubiquitous in ordered (and at first glance transitive) relations in complex mathematical systems and the real world. Relatedly, \cite{PAHIKKALA2010676} previously developed a specific kernel method for learning intransitive relations. While they focused on empirically fitting benchmark datasets, we analytically characterize the out-of-distribution generalization performance arising from a very similar kernel approach that can be used to characterize the generalization behavior of a random features model.

\textbf{Benign overfitting.}
The closest concept in classical in-distribution generalization to exceptions is label noise. While memorizing label noise does not have as clear benefits as memorizing exceptions, many modern machine learning algorithms nevertheless overfit their training data. Intriguingly, while classical perspectives suggest that overfitting makes generalization worse, recent work has identified certain regimes (including ridge regression with certain types of kernels) in which networks overfit their data without becoming worse at generalization (``benign overfitting'') \cite{bartlett2020benign,chatterji2022interplay,mallinar2022benign,tsigler2023benign,barzilai2026beyond}. Our work here emphasizes that ``almost additive'' representation exhibit benign memorization of exceptions, but only for no explicit regularization.
\section{Encoding Model and Mechanistic Intuition}
\label{sec:encoding}
\subsection{Encoding Model}
We define a four-hot representation that represents any item $x_{j,k}$ as 
\begin{equation}
    z=\begin{pmatrix}
        s_1z_j^T&s_1z_k^T&s_2z_{j,k}^T&s_0z_b
    \end{pmatrix}\in\mathbb{R}^{2n+n^2+1},
\end{equation}
where $z_j\in\mathbb{R}^n$ is the one-hot vector with an entry in the $j$-th position and $z_{j,k}$ is a one-hot vector with an entry in the $(j-1)n+k$-th position (or more intuitively, the one-hot vector representing each specific combination of items). Finally $z_b=1$. For a representation $H\in\mathbb{R}^{n^2\times h}$ of our pairs of items with associated $k_d,k_o,k_s$, we set
\begin{equation}
    s_0=\sqrt{k_d},\quad s_1=\sqrt{k_o-k_d},\quad s_2=\sqrt{(k_s-k_d)-2(k_o-k_d)}.
\end{equation}
We write the complete feature matrix of $z$ as
\begin{equation}
    Z\in\mathbb{R}^{n^2\times 2n+n^2+1}
\end{equation}
Then, $HH^T$, i.e.\ the kernel induced by $H$, is equivalent to the kernel induced by the four-hot representation, $ZZ^T$. The original representation and the four-hot representation have the same inner-product geometry over the represented inputs (note that this is a direct consequence of the fact that both matrices have identical left-singular vectors and singular values):
\begin{equation}
    H=ZR,\quad Z=HR^T,
\end{equation}
where $R\in\mathbb{R}^{(n^2\times 2n+n^2+1)\times h}$ is an orthonormal matrix. (Note that $R$ is not square and we mean more specifically a matrix with orthogonal rows if $h\geq(n^2\times 2n+n^2+1)$ and orthogonal columns if $h\leq(n^2\times 2n+n^2+1)$.)
\subsection{Mechanistic Intuition}
\label{app:intuition}
We note that we can rewrite readouts from $H$, as readouts from $Z$:
\begin{equation}
    f(x_{j,k})=H_{j,k}w=Z_{j,k}\tilde{w}\mbox{ for }\tilde{w}:=Rw.
\end{equation}
Because $R$ is an orthonormal matrix in the sense defined above,
\begin{equation}
    \|\tilde{w}\|_2^2=\|w\|_2^2.
\end{equation}
We now split up $\tilde{w}$ into the ranking system, the conjunctive coefficient and the bias:
\begin{equation}
    \tilde{w}=\begin{pmatrix}
        \tilde{r}^{(1)T}&\tilde{r}^{(2)T}&\tilde{t}^T&b
    \end{pmatrix}.
\end{equation}
We note that by the same symmetry arguments used in our proof, $\tilde{r}^{(1)}=-\tilde{r}^{(2)}=:\tilde{r}$ and $b=0$. Thus, we can write
\begin{equation}
    f(x_{j,k})=s_1(\tilde{r}_j-\tilde{r}_k)+s_2\tilde{t}_{j,k}.
\end{equation}
We now set
\begin{equation}
    r_j:=s_1\tilde{r}_j,\quad t_{j,k}:=s_2\tilde{t}_{j,k}.
\end{equation}
Conveniently, this allows us to express the costs on $w$ in terms of costs on the effective ranking system and the conjunctive coefficient:
\begin{align}
\begin{split}
    \|w\|_2^2=s_1^{-2}\|r\|_2^2+s_2^{-2}\|t\|_2^2&=\frac{1}{k_o-k_d}\|r\|_2^2+\frac{1}{(k_s-k_d)-2(k_o-k_d)}\|t\|_2^2\\&\propto\frac{k_s-k_d}{k_o-k_d}\|r\|_2^2+\frac{k_s-k_d}{(k_s-k_d)-2(k_o-k_d)}\|t\|_2^2.
\end{split}
\end{align}
This means that
\begin{equation}
    \|w\|\propto C(r,t):=\frac{2}{1-\alpha}\|r\|_2^2+\frac{1}{\alpha}\|t\|_2^2,
\end{equation}
which demonstrates that we can express the weight norm in terms of a relative weight on the ranking system and the conjunctive coefficients. This allows us to compare the relative incentives enforced by these different terms, as in Fig.~\ref{fig:intuition} and Section~\ref{sec:intuition}.
\begin{figure}
    \centering
    \includegraphics{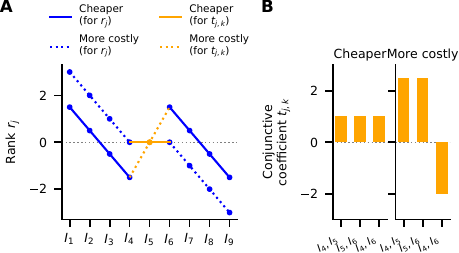}
    \caption{Illustration of the intuition. \textbf{A}, Two possible ranking systems, one of which is more costly and one of which is cheaper. \textbf{B} The values of the conjunctive coefficient, in order to encode the intransitive loop given the two different ranking systems. If we prioritize costs on the ranking system, we want to center the ranks of the section before and after the intransitive loop (solid blue line as opposed to dotted blue line). Thus, wanting to minimize the norm of the ranking system confers an inductive bias towards non-monotonicity. However, this means that the conjunctive coefficients need to counteract the ranking systems predictions, which incurs a substantially larger norm. To minimize the norm of the conjunctive coefficients, we want to keep the intransitive section flat (solid orange line as opposed to dotted orange line). This confers successful cross-section generalization. Because in practice we optimize a mixture of the ranking system and memorization cost, we always arrive at a non-monotonic ranking system, albeit a more attenuated one. A more compositional representations means that the conjunctive coefficient is more costly and hence the model prioritizes a flat ranking system. A more conjunctive representation means that the ranking system is more costly and hence the model prioritize a more non-monotonic ranking system.}
    \label{fig:intuition}
\end{figure}
\section{Experiments with Pretrained Language Models}
\label{app:dnns}
\subsection{Detailed Setup}
\subsubsection{Task modalities.}
We consider four task modalities: sports teams (either real or fake) and poker hands (either correct or reversed).

\paragraph{Sports teams.} In this case, the relation is whether team A is likely to win against team B, presented as ``Q: Will the [team A] win against the [team B]? A: [Yes/No].'' To ensure that the models are biased to responding with Yes or No rather than a different format even before pretraining, we provide two ``anchor questions'' that indicate the format. To this end, we randomly sample a winning and a losing team (neither of which are a part of the actual hierarchy) and present them as additional preceding question. In total, this gives rise to the following format:
\begin{quote}
    Q: Will the [anchor winner] win against [anchor loser]? A: Yes.\\
    Q: Will the [anchor loser] win against [anchor winner]? A: No.\\
    Q: Will the [team A] win against [team B]? A: [Yes/No].
\end{quote}
We sampled teams from one of two lists. To minimize the degree of prior expectation on this relation, we considered sports teams with fictional names (``fake teams''). To control for any potential confusion arising from such fictional names, we additionally also considered real MLB teams (``real teams''). Below we list all names used in either of these modalities:
\begin{itemize}
    \item \textbf{Fake teams:} The Harlem Renaissance, Marysville Mudskippers, Ogden Railspikes, Cheyenne Thunderbirds, Colorado Springs Quake, Fort Huachuca Javelina, Honolulu Pu'ali, Cambridge Charlies, Bethesda Warhawks, Grand Junction Dunes, Biloxi Shrimp, Ocean Springs Sandhill Crane, Gulfport Swampmen, Everett Yeti, McLean Majors, Morningside Midnight Run, Pawnee Fire, Carson City Cougars, Boise Noise, Silverton Slides, Durango Stampede, Ouray Morays, Munich Monks, Hamburg Raiders, Cologne Horribles, Mount Erie Field Devils, Waynesboro Knights, Oakland Visigoths, Gainesville Gators, Des Moines Shuckers
    \item \textbf{Real teams:} Boston Red Sox, New York Yankees, Toronto Blue Jays, Chicago White Sox, Cleveland Guardians, Detroit Tigers, Kansas City Royals, Minnesota Twins, Houston Astros, Los Angeles Angels, Oakland Athletics, Seattle Mariners, Texas Rangers, Atlanta Braves, Miami Marlins, New York Mets, Philadelphia Phillies, Washington Nationals, Chicago Cubs, Cincinnati Reds, Milwaukee Brewers, Pittsburgh Pirates, St. Louis Cardinals, Arizona Diamondbacks, Colorado Rockies, Los Angeles Dodgers, San Diego Padres, San Francisco Giants
\end{itemize}
\paragraph{Poker hands.} The possible hands are any pair of two cards, represented by their ranks (written in order as A, K, Q, J, T, 9,..., 2) as well as whether they have matching suits (indicated by a ``s'') or not (indicated by a ``o''). If the two ranks are matched, we know that their suits are not matched and we don't specify that further. For example an ace and a queen with different suits as represented as AQo. We note that this is standard poker terminology and we therefore expect that the language models are familiar with it. We consider as our relation whether hand A is more likely to win against hand B in a one-on-one match-up (``winningness'') (see Appendix~\ref{sec:poker}). We present the relation in the following format:
\begin{quote}
    We consider pre-flop all-in heads-up.\\
    Q: Is AA likely to win against 72o? A: Yes.\\
    Q: Is 72o likely to win against AA? A: No.\\
    Q: Is [hand A] likely to win against [hand B]? A: [Yes/No].
\end{quote}
Unlike for the sports team, we use the ground truth (according to an equity matrix we generated using \texttt{eval7}, see Appendix~\ref{sec:poker}). We either consider that relation directly, or invert it (i.e.\ such that worse hands are indicated as winning against better hands, ``poker (reversed)''). This allows us to evaluate the degree to which real-world information influences the pretrained language models' judgment. To sample the relevant relations, we specifically sample cases where the winning hand has a probability between 51\% and 60\% of winning against the losing hand, so as to make the task sufficiently challenging. For TI, we make sure to sample an actually transitive hierarchy, checking each pairwise relation. In contrast, for TI with exceptions, we only control for the training set pairs. We sample these sets by first sampling a particular hand and then sequentially generating the next hand such that it satisifes the constraints of the hierarchy. If no hand satisfies those constraints, we consider a different sample of starting hand and start again. We confirmed that models had some (but not much prior knowledge) about the poker hands (Fig.~\ref{fig:acc-pre}).
\begin{figure}
    \centering
    \includegraphics[width=\linewidth]{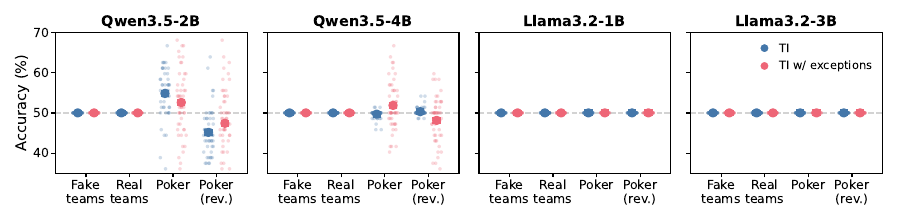}
    \caption{Accuracy for all models prior to finetuning for all four tasks. Line ranges represent 95\% binomial confidence intervals (may be too small to see).}
    \label{fig:acc-pre}
\end{figure}
\paragraph{Task hierarchies}
We consider four task hierarchies: TI ($n=9$) and TI with exceptions ($n=9,p=6,q=4$, $n=11,p=7,q=5$, and $n=11,p=8,q=4$). In total, this means that we consider sixteen different tasks (4 task hierarchies combined with four task modalities).
\paragraph{Evaluation}
We consider $n=100$ randomly generated datasets for the real and fake sports teams and $n=50$ randomly generated datasets for the regular and reversed Poker hierarchies. To evaluate the models, we both compute the probability on each training and generalization pair that the assign on ``yes'' vs. ``no.'' For the cases where we have a ground truth, we additionally determine whether the model was correct (i.e.\ assigned a higher probability to the correct answer). We compute accuracies by averaging over all random seeds.
\subsubsection{Model setup}
We consider three finetuning strategies: full finetuning, low-rank adaptation (LoRA), and linear probing. In all cases, we use AdamW \cite{loshchilov2018decoupled} with parameters $\beta_1=0.9,\beta_2=0.999$ and a weight decay of $10^{-2}$ (as is default in the PyTorch implementation). We emphasize that our objective here was not to find the hyperparameters that worked best, but rather to use representative parameters to understand the default behavior arising from finetuning. We train on crossentropy over the yes/no tokens specifically and do not use the rest of presented sequence as part of our objective.

We consider four recent pretrained language models: Llama3.2-1B/3B and Qwen3.5-2B/4B. We only apply full finetuning to the smaller two models, giving rise to ten total finetuning setups.
\paragraph{Full finetuning.}
We finetune all weights in the model, with a learning rate of $10^{-6}$ (for Qwen3.5-2B) and $2\cdot 10^{-6}$ (for Llama3.2-1B). We use a batch size of 32 and finetune for 100 epochs.
\paragraph{Low-rank adaptation (LoRA).} We apply LoRA to all parameters in the attention and MLP layers. We separately consider ranks $r=4,8,16,32,64$ and set $\alpha=2r$. By default, we consider $r=32$, but we analyze the effect of different ranks below. We use a batch size of 16 and finetune for 10 epochs.
\paragraph{Linear probing.} We only train the linear readout head of the model, using $L_2$-regularization with respect to the readout weights. For the inverse scaling factor of the regularization $c$, we use the values $c=10^4,10,1,0.1,0.05,0.02$. By default we consider extremely small regularization, $c=10^4$. We use a batch size of 16 and train for 50 epochs.

\subsubsection{Reproducibility}
The code for reproducing the experiments described in Section~\ref{sec:dnns} can be found here: \url{https://github.com/sflippl/ti-with-exceptions}. We ran all experiments on a Slurm cluster using A40 GPUs. Each experiment (which involved running 100 seeds) took between 2--6h. We ran a total of 464 experiments requiring between 1000-2700h of compute in total.
\subsection{Transitive Inference}
\label{app:ti}
We first trained models on transitive inference ($n=9$). To evaluate model behavior on this task, we draw a distinction between three types of pairs: adjacent items (which the model was trained on), non-adjacent items that include a terminal item (``Gen. (terminal)'') and non-adjacent items that only include internal items (``Gen. (internal)''). This is motivated by the fact that the former generalization set can be solved by simply learning to always indicate that the first item is larger and that the last item is smaller, without requiring actually relating the items to each other. In contrast, internal items are necessarily context-dependent in their relation. Thus, transitive generalization is more accurately captured by the ability to transitively generalize on the internal items.
\paragraph{Full finetuning yielded successful transitive generalization.}
Figure~\ref{fig:ti-acc-full} depicts the accuracy of models fully finetuned on TI (further, see Fig.~\ref{fig:ti-summary}A, top row, for summarized performance). We find that the models broadly generalize above chance across all task modalities. 
\begin{figure}
    \centering
    \includegraphics{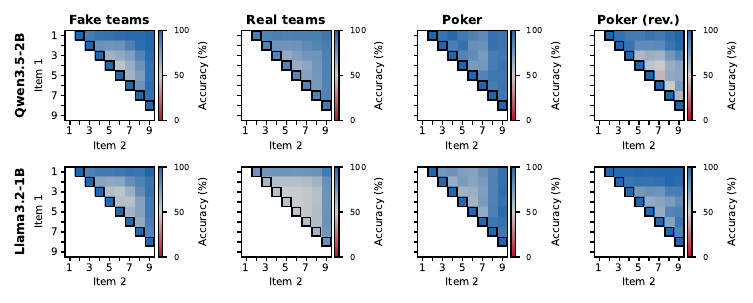}
    \caption{Accuracy for all models and tasks for full finetuning.}
    \label{fig:ti-acc-full}
\end{figure}
\paragraph{LoRA.}
We found that low-rank adaptation largely yielded more modest gains in performance. For most cases, transitive generalization (including on the internal items) was still above chance, however. We note that LoRA performed particularly badly on Llama3.2-1B. It is possible that further hyperparameter optimization or longer training would yield stronger benefits from LoRA.

We then evaluated the impact of varying rank on generalization (Fig.~\ref{fig:ti-summary}B). We found that intermediate ranks (16 or 32) tended to perform best. However, for no parameter choices did the models exhibit systematic errors (i.e.\ below-chance accuracy). This is consistent with the fact that kernel models also never make such systematic errors.
\begin{figure}
    \centering
    \includegraphics{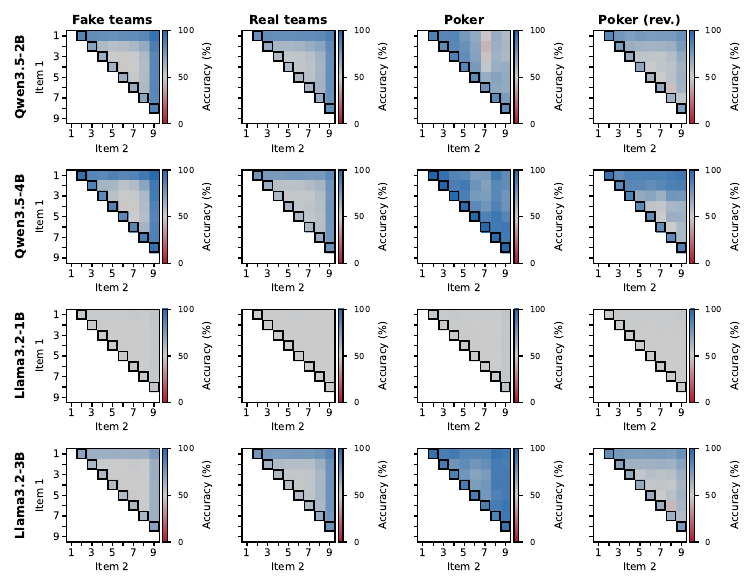}
    \caption{Accuracy for all models finetuned with LoRA (rank=32).}
    \label{fig:ti-acc-lora}
\end{figure}
\paragraph{Linear probing.}
Finally, we again observed strong transitive generalization for linear probing (Fig.~\ref{fig:ti-summary}C). We found that stronger generalization decreased performance on both training and generalization. Again, model performance never went below chance.
\begin{figure}
    \centering
    \includegraphics{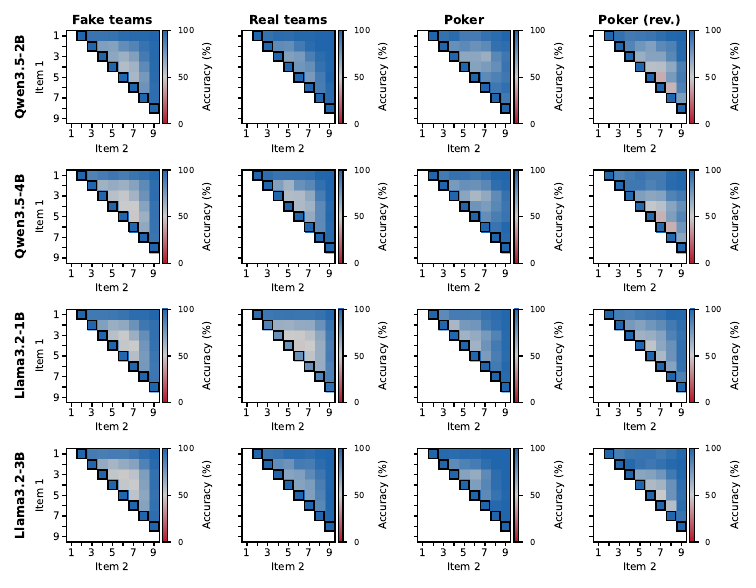}
    \caption{Accuracy for all models finetuned with the linear probe ($c=10^4$).}
    \label{fig:ti-acc-probe}
\end{figure}
\subsection{Consistent behavioral effects on transitive inference}
\label{app:behavioral-effects}
On both full finetuning (Fig.~\ref{fig:ti-acc-full}) and linear probing (Fig.~\ref{fig:ti-acc-probe}), we observed that generalization on the sports teams becomes better as a function of the symbolic distance between the two items (i.e.\ the distance of their ranks). Moreover, the models are also better at relatively more terminal items than more internal items. Notably, all of these behavioral effects parallel effects robustly observed in humans and animals as well \cite{vasconcelos2008transitive,jensen2017serial}. On the poker task modalities, behavior in the case of full finetuning looks a little different and models are better at generalizing for later items in the hierarchy (or, identically, earlier items in the reversed hierarchy). We speculate that this is a consequence of the prior knowledge the language models have about poker. In contrast, for linear probing, we still observe the same behavioral effects for the poker modality, indicating that prior knowledge affects linear probing differently from full finetuning.

\begin{figure}
    \centering
    \includegraphics{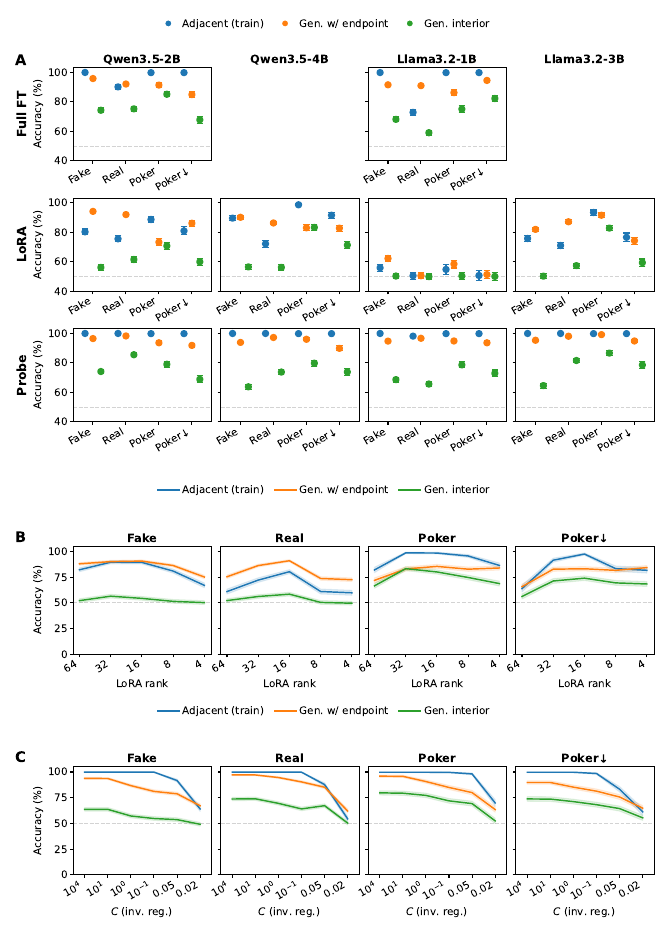}
    \caption{Summary of behavior on TI. Shaded regions/line ranges represent 95\% binomial confidence intervals (may be too small to see). \textbf{A}, Accuracy before and after finetuning across all tasks and for different finetuning methods. \textbf{B}, Accuracy on LoRA as a function of rank. \textbf{C}, Accuracy on the linear probe as a function of the regularization coefficient.}
    \label{fig:ti-summary}
\end{figure}
\subsection{Transitive inference with exceptions (base case)}
\label{app:exc}
To complement Fig.~\ref{fig:dnns}D in the main text, Figs.~\ref{fig:exc-acc-full},\ref{fig:exc-acc-lora},\ref{fig:exc-acc-probe} depict generalization performance across the different finetuning methods, task modalities, and models. As indicated by Fig.~\ref{fig:dnns} within-section generalization is highly reliable, whereas cross-section generalization is variable and sometimes even below chance. We again observe that LoRA performs much worse on the task. We summarize the overall performance in Fig.~\ref{fig:task-sweep}, which shows, in particular, that models are substantially worse at cross-section generalization than within-section generalization.
\begin{figure}
    \centering
    \includegraphics{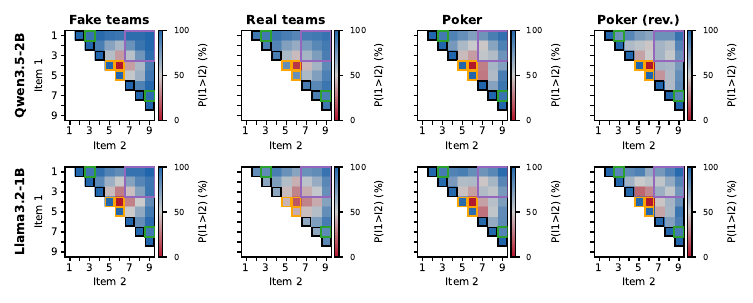}
    \caption{Model predictions for all tasks and models for full finetuning.}
    \label{fig:exc-acc-full}
\end{figure}
\begin{figure}
    \centering
    \includegraphics{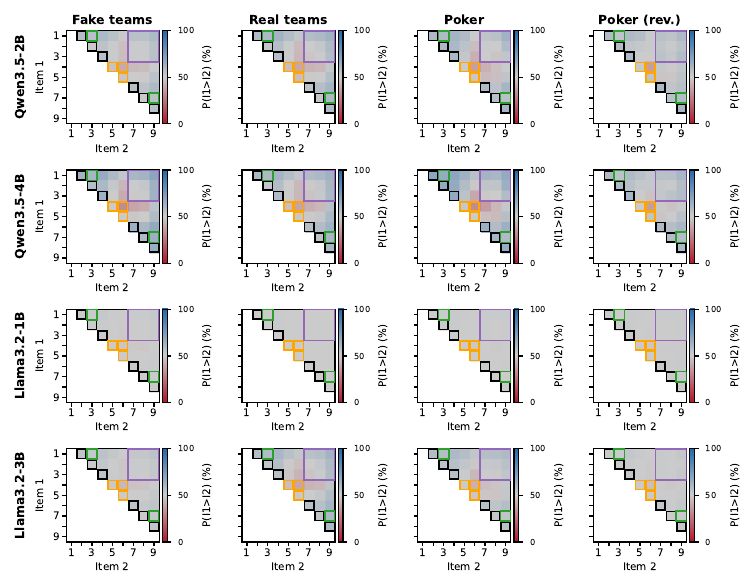}
    \caption{Model predictions for all tasks and models for LoRA (rank=32).}
    \label{fig:exc-acc-lora}
\end{figure}
\begin{figure}
    \centering
    \includegraphics{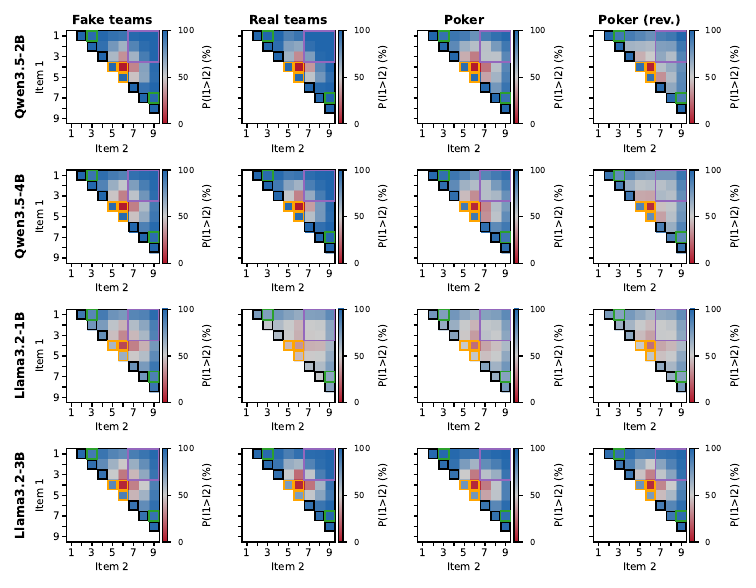}
    \caption{Model predictions for all tasks and models for linear probing ($c=10^4$).}
    \label{fig:exc-acc-probe}
\end{figure}
\begin{figure}
    \centering
    \includegraphics{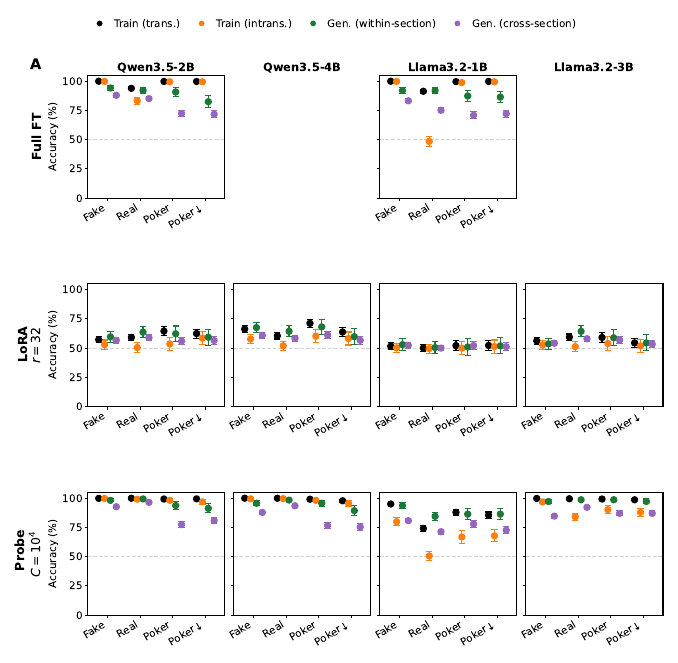}
    \caption{Accuracy on the different training and generalization splits, across tasks, models, and finetuning methods. Line ranges represent 95\% binomial confidence intervals (may be too small to see). }
    \label{fig:task-sweep}
\end{figure}

To complement Fig.~\ref{fig:dnns}E, we find that across all models and task modalities, training performance on $I_4$ vs.\ $I_5$ and $I_5$ vs.\ $I_6$ is lower than on $I_4$ vs.\ $I_6$, as predicted by our theory (Fig.~\ref{fig:exc-sweep}). Moreover only the former two data points are ever below chance---again as predicted by our theory. We note that on the linear probe, the regularization coefficient behaves highly similar to the predictions made by the kernel model: as regularization becomes stronger, performance on the training data points becomes worse. (This is perhaps not too surprising given the close analogy between these models.) In contrast, rank in LoRA does not have such a consistent effect. Rather, performance often becomes better for lower rank. Thus, the effects of varying the rank in LoRA are not captured by our theoretical model. However, we emphasize that the coarse trends we highlighted above still hold true.
\begin{figure}
    \centering
    \includegraphics{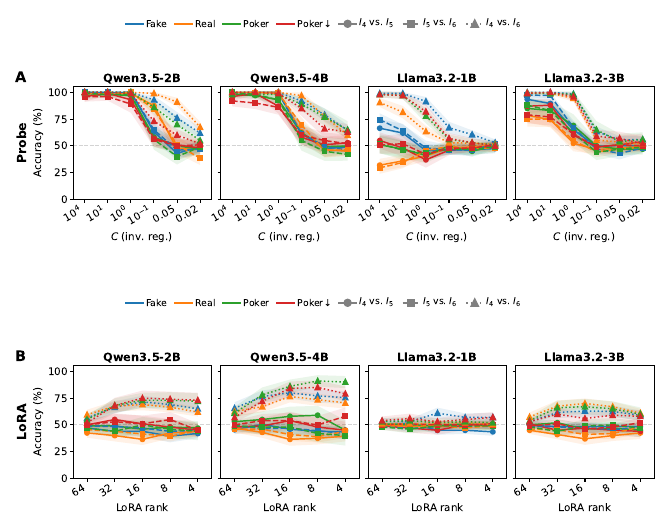}
    \caption{Same plot as Fig.~\ref{fig:dnns}E across different task modalities, models, and finetuning methods. Specifically, we vary $c$ (for the linear probe) and we vary the rank (for LoRA). We then compute accuracy for the three distinct intransitive pairs. Shaded regions represent 95\% binomial confidence intervals (may be too small to see). }
    \label{fig:exc-sweep}
\end{figure}

Finally, to complement Fig.~\ref{fig:dnns}F, we find that across all models and task modalities, generalization performance within-section is better than cross-section (Fig.~\ref{fig:exc-sweep}). Moreover, only cross-section generalization sometimes goes below chance. On linear probing, $c$ generally has the effect predicted by our theory: both within-section and cross-section generalization performance decreases for stronger regularization and for cross-section generalization, performance can even intermittently go below chance.
\begin{figure}
    \centering
    \includegraphics{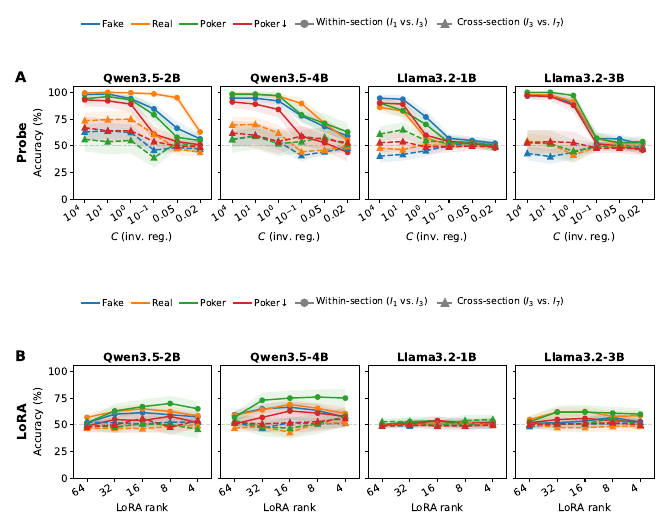}
    \caption{Same plot as Fig.~\ref{fig:dnns}F across different task modalities, models, and finetuning methods. Specifically, we vary $c$ (for the linear probe) and we vary the rank (for LoRA). We then compute accuracy for the within-section generalization pair and the most challenging cross-section generalization pair. Shaded regions represent 95\% binomial confidence intervals (may be too small to see).}
    \label{fig:gen-sweep}
\end{figure}
\subsection{Varying task parameters}
To expand on Fig.~\ref{fig:dnns}G,H, Figs. \ref{fig:exc-mid} and \ref{fig:exc-far} depict the generalization behavior on the two additional task variants.
\begin{figure}
    \centering
    \includegraphics{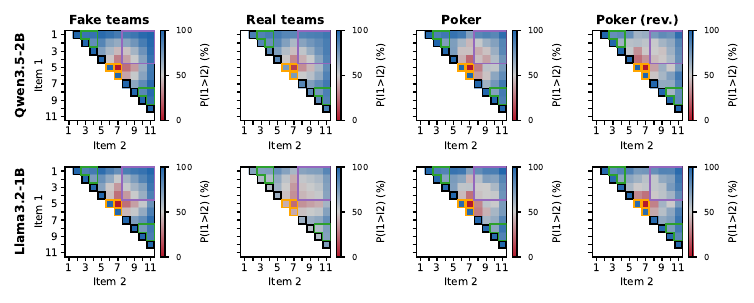}
    \caption{Accuracy on full finetuning for more transitive items.}
    \label{fig:exc-mid}
\end{figure}
\begin{figure}
    \centering
    \includegraphics{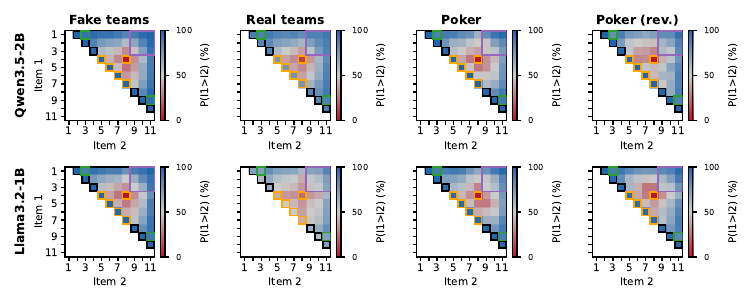}
    \caption{Accuracy on full finetuning for more intransitive items.}
    \label{fig:exc-far}
\end{figure}
\subsection{Content effects in the poker modality}
\label{sec:content-effects}
Humans and PLMs seem to draw logical inferences that are consistent with their prior knowledge more easily, than ones that are inconsistent \cite{evans1983conflict,evans1995belief,dasgupta2022language,lampinen2024language}. We generally found that this was true for the PLMs finetuned on TI as well: models trained on poker tended to perform better than models trained on the reverted version of poker (Fig.~\ref{fig:ti-summary}A). A surprising exception is given by Llama3.2-1B, which showed the opposite effect and performed better in its generalization on the reversed poker hierarchy. For all other finetuning setups, however, we consistently found that models were worse at generalizing over the reversed poker hierarchy than the non-reversed poker hierarchy. Overall, we therefore found that models, in almost all cases of TI, showed a content effect. In contrast, we did not consistently observe a content effect with TI with exceptions. This highlights that introducing more complicated relational or logical rules that potentially contain exceptions as well may differentially affect content effects in PLMs.
\subsection{Open questions}
Overall, we observed that our theoretical model provided a striking account of the finetuning behavior of pretrained language models. Below, we expand on two questions our empirical experiments have raised, that our theory cannot address.

\paragraph{How does prior knowledge impact transitive generalization?}
In practice, relational generalization will interact with prior knowledge. In particular, on TI we observed that models performed better on the poker hierarchy than the reversed poker hierarchy (see Section~\ref{sec:content-effects}). Our theory currently cannot account for such effects of prior knowledge: deviating from the exchangeability assumption might allow us to address this point. This could potentially help us better understand why these content effects appear to be less salient for TI with exceptions.

\paragraph{How can we understand the role of rank for relational generalization through LoRA?}
We found that the rank in low-rank adaptation played a different role from the regularization coefficient in ridge regression (though both are commonly understood to regulate the amount of changes induced by the network). This highlights that a different theoretical paradigm may be necessary to help us more closely understand the behavior of LoRA on relational tasks. However, we note that our theory nevertheless captured the coarse-grained structure of behavior in LoRA.
\section{An Example of an Intransitive Relation: Winningness of Poker Hands}
\label{sec:poker}
To provide an illustration of an intransitive relation, we consider Texas Hold'em poker. Specifically, we consider the relation where for two hands, $x\succ y$ iff $x$ is more likely to win against $y$ in a heads-up match. We computed the equity of all match-ups between hands using \texttt{eval7} and then sampled sets of hands that adhered to the structure of TI with exceptions ($n=9,p=6,q=4$). We then computed the frequency with which $x\succ y$ for any pair of hands (Fig.~\ref{fig:poker}). We found that in most cases, within-section and cross-section generalization would result in the appropriate generalization in this case.
\begin{figure}
    \centering
    \includegraphics{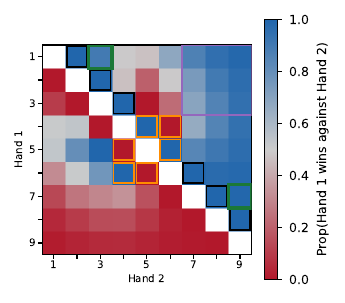}
    \caption{Illustration of TI with exceptions in a real-world example. We sampled sets of hands in poker such that the winningness relation (i.e.\ for two hands, $x\succ y$ iff $x$ is more likely to win against $y$) adhered to the training set in TI with exceptions. Under those constraints, we determined the proportion with which hand 1 (represented along the row) is more likely to win against hand 2. This analysis highlights that both cross-section generalization and within-section generalization would, on average result in the correct generalization.}
    \label{fig:poker}
\end{figure}

\section{Proof of the main theorem}
\label{app:proof}
We first state the full version of the theorem:
\begin{theorem}[Theorem~\ref{thm} (detailed).]
    On TI with exceptions with parameters $n,p,q$ and the training dataset $\mathcal{D}$, the predictions of a kernel model with an exchangeable representation are given by
    \begin{equation}
        f(x_{j,k})=\begin{cases}
            r_j(\alpha,\tilde{c})-r_k(\alpha,\tilde{c})&\text{ for }x\notin\mathcal{D},\\
            m(\alpha,\tilde{c})y_{i,j}+(1-m(\alpha,\tilde{c}))(r_i(\alpha,\tilde{c})-r_j(\alpha,\tilde{c}))&\text{ for } x\in\mathcal{D},
        \end{cases}
    \end{equation}
    where
    \begin{align}
        m(\alpha,\tilde{c})&:=\tfrac{\alpha}{\alpha+\tilde{c}^{-1}},\quad r_j(\alpha,\tilde{c}):=r_j\TI(\alpha,\tilde{c})+r_j^{\mathrm{pert}}(\alpha,\tilde{c}),\\
        r_j\TI(\alpha,\tilde{c})&:=\frac{\sinh((\frac{n+1}{2}-j)\lambda)}{\sinh(\frac{n+1}{2}\lambda)-\sinh(\frac{n-1}{2}\lambda)},\quad\lambda := \operatorname{arccosh}\!\left(\frac{1+\tilde{c}^{-1}}{1-\alpha}\right),\label{eq:rti}\\
        r_j^{\mathrm{pert}}(\alpha,\tilde{c})&:=\frac{\left(1 - {r}_p\TI + {r}_q\TI\right)\left(\tilde{D}_{jp} - \tilde{D}_{jq}\right)}{\sinh(\lambda)\sinh(n\lambda) - 2\sinh\left(\frac{q-p}{2}\lambda\right) G},\label{eq:rpert}\\
        G &:= \cosh\left(\left(q - \tfrac{1}{2}\right)\lambda \right)\sinh\left(\left(n - \tfrac{p+q-1}{2}\right)\lambda \right) + \cosh\left(\left(n - p + \tfrac{1}{2}\right)\lambda \right)\sinh\left(\tfrac{p+q-1}{2}\lambda \right),\\
        \tilde{D}_{ij}&:=\cosh\bigl((\min(i,j)-\tfrac{1}{2})\lambda\bigr)\cosh\bigl((n-\max(i,j)+\tfrac{1}{2})\lambda\bigr).
    \end{align}
    We note that $r\TI_j$ is the ranking system for TI without an exception and hence $r_j^{\text{pert}}$ represents the perturbation on top of it.
    
    The above definition is only valid for $\alpha<1$. For $\alpha=1$,
    \begin{equation}
        r_j(\alpha,\tilde{c})=0.
    \end{equation}
\end{theorem}
\subsection{Prerequisites}
\begin{theorem}[Hu \& O'Connell \cite{hu1996analytical}]
\label{thm:tri-inv}
Let $a,b\in\mathbb{R}$ such that $\tfrac{a}{b}\leq-2$ and
\begin{equation}
    B\in\mathbb{R}^{m\times m},\quad
    B_{ij}=\begin{cases}
    a&\text{ if }i=j,\\
    b&\text{ if }|i-j|=1,\\
    0&\text{ else}.
    \end{cases}
\end{equation}
Then
\begin{equation}
    B^{-1}_{ij}=-\frac{\cosh\left((m+1-|j-i|)\lambda\right)-\cosh\left((m+1-i-j)\lambda\right)}{2b\sinh\left(\lambda\right)\sinh\left((m+1)\lambda\right)},
\end{equation}
where
\begin{equation}
    \lambda:=\mathrm{arccosh}\left(-\frac{a}{2b}\right).
\end{equation}
\end{theorem}
\subsection{Setup}
\label{app:ti_setup}
We consider $n\in\mathbb{N}$ items $(x_j)_{i=1,\dotsc,n}$, $x^{(i)}\in\mathbb{R}^d$. The task presents two of those items together and we denote this by
\begin{equation}
    x_{(j,k)}=\begin{pmatrix}x_j\\x_k\end{pmatrix}\in\mathbb{R}^{2d}.
\end{equation}

The corresponding label is given by
\begin{equation}
    y_{j,k}=\begin{cases}+1 & \text{if } (j,k)=(p,q),\\
                          -1 & \text{if } (j,k)=(q,p),\\
                          +1 & \text{else, if } j<k,\\
                          -1 & \text{else, if } j>k,\end{cases}
\end{equation}
where $(p,q)$ with $1\leq q<p\leq n$ and $p-q\geq 2$ denotes the designated \emph{exception pair}. Since $q<p$, the transitive rule would assign $y_{p,q}=-1$ and $y_{q,p}=+1$; the exception flips both, making this pair inconsistent with the underlying ranking.

The training dataset consists of all adjacent items together with the exception pair, i.e.
\begin{equation}
    \mathcal{D}=\underbrace{\left\{\bigl(x_{j,j+1},1\bigr),\bigl(x_{j+1,j},-1\bigr)\,\middle|\,j=1,\dotsc,n-1\right\}}_{\text{adjacent pairs}}\cup\underbrace{\left\{\bigl(x_{p,q},+1\bigr),\bigl(x_{q,p},-1\bigr)\right\}}_{\text{exception pair}}.
\end{equation}

In matrix form, we denote this dataset as
\begin{equation}
    X=\begin{pmatrix}x_{1,2}\\\vdots\\x_{n-1,n}\\\hline x_{2,1}\\\vdots\\x_{n,n-1}\\\hline x_{p,q}\\x_{q,p}\end{pmatrix}\in\mathbb{R}^{(2n)\times 2d},\quad
    y=\begin{pmatrix}+1\\\vdots\\+1\\\hline -1\\\vdots\\-1\\\hline +1\\-1\end{pmatrix}\in\mathbb{R}^{2n}.
\end{equation}

We consider a set of features
\begin{equation}
    \phi:\mathbb{R}^{2d}\to\mathbb{R}^h,\quad h\in\mathbb{N},
\end{equation}
and a linear model of those latent features that predicts the output, i.e.
\begin{equation}
    \hat{y}_{i,j}=\langle w,\phi(x_{i,j})\rangle.
\end{equation}
\subsection{General intuition}
\label{app:ti_intuition}
Many standard statistical learning models (in particular the ones we consider below) only depend on their input representation through its trial-by-trial similarity (or ``induced kernel'') 
\begin{equation}
    K(x,x')=\langle\phi(x),\phi(x')\rangle.
\end{equation}
In particular, the ridge regression solution (\ref{eq:opt}) can be described in terms of the ``dual coefficients''
\begin{equation}
    a=(K+\tfrac1cI)^{-1}y,
    \label{eq:dual}
\end{equation}
where $K$ is the trial-by-trial similarity of the training dataset:
\begin{equation}
    K=\phi(X)^T\phi(X)\in\mathbb{R}^{2n\times 2n}.
\end{equation}
Each dual coefficient corresponds to a particular training data point $(j,k)$ and we denote it by $a_{(j,k)}$. The dataset now contains $2n$ training points: the $2(n-1)$ adjacent pairs plus the two exception pairs $(p,q)$ and $(q,p)$. The model behavior on a new data point $x_{j,k}$ is then given by
\begin{equation}
    f(x_{j,k})=\langle k_{(j,k)}^{(\text{test})},a\rangle,\quad k_{(j,k)}^{(\text{test})}:=\phi(X)^T\phi(x_{j,k}).
\end{equation}
Importantly, the similarity between test data points and training data points can only take on two possible values: $\kappa_o$ and $\kappa_d$. The training points whose similarity to $(j,k)$ equals $\kappa_o$ are exactly those that share an item with $(j,k)$. Among the adjacent pairs, these are $a_{(j,j+1)}$, $a_{(j,j-1)}$, $a_{(k-1,k)}$, and $a_{(k+1,k)}$. The exception pairs $a_{(p,q)}$ and $a_{(q,p)}$ additionally contribute with similarity $\kappa_o$ whenever $j$ or $k$ coincides with $p$ or $q$. (If $j,k\in\{1,n\}$, some adjacent-pair coefficients will not exist; we set them to zero. For example, if $i=0$, $a_{(i,i-1)}=a_{(1,0)}$ corresponds to a non-existent data point and is set to zero.) All other training points have similarity $\kappa_d$. Taken together, we can express model behavior as
\begin{equation}
\begin{split}
    f(x_{j,k})=\kappa_d\sum a&+(\kappa_o-\kappa_d)\bigl(a_{(j,j+1)}+a_{(j,j-1)}+a_{(k-1,k)}+a_{(k+1,k)}\bigr)\\
    &+(\kappa_o-\kappa_d)\bigl((\delta_{j=p}+\delta_{k=q})a_{(p,q)}+\delta_{j=q}+\delta_{k=p})a_{(q,p)}\bigr).
\end{split}
\end{equation}
The first line is identical to the no-exception case: a sum of values dependent on $i$ and values dependent on $j$, which already implements a ranking system. The second line is the exception correction: it activates only when the test pair shares an item with the exception pair, and it too decomposes additively into an $j$-dependent and $k$-dependent contribution. We can therefore define a generalized ranking system
\begin{equation}
    r_1(j)=(\kappa_o-\kappa_d)\bigl(a_{(j,j+1)}+a_{(j,j-1)}\bigr)+(\kappa_o-\kappa_d)\,\delta_{j=p}a_{(p,q)}+\delta_{j=q}a_{(q,p)},
\end{equation}
\begin{equation}
    r_2(k)=-(\kappa_o-\kappa_d)\bigl(a_{(k-1,k)}+a_{(k+1,k)}\bigr)-(\kappa_o-\kappa_d)\,\delta_{k=q}a_{(p,q)}+\delta_{k=p}a_{(q,p)},
\end{equation} 
yielding
\begin{equation}
    f(x_{j,k})=r_1(j)-r_2(k).
\end{equation}

To see whether the model generalizes, we need to determine whether the ranks are ordered correctly. To do so, we analytically solve the inverse problem Eq.~\eqref{eq:dual}. This is described in detail in the next section. As a brief summary, we rely on previous results on inverses for banded diagonal matrices \cite{hu1996analytical} to compute the inverse matrix. This inverse is expressed in terms of hyperbolic functions. We then use hyperbolic identities in a few different ways to compute $a$ and, finally, $r_1$ and $r_2$ from this. This reveals that $r_1(i)=r_2(i)$ and further gives rise to the analytical expression provided in the main text.

\subsection{Expression of the dual solution.}
Here we assume that the weights are learned through ridge regression. Specifically, our loss function is given by the mean squared error over the training set:
\begin{align}
L(w)
&= \sum_{(x,y)\in \mathcal{D}} \ell\!\bigl(y, \langle w, \phi(x)\rangle\bigr) \nonumber \\
&= \sum_{j=1}^{n-1} \ell\!\left(1, \bigl\langle w, \phi(x_{j,j+1})\bigr\rangle\right)
 + \sum_{j=1}^{n-1} \ell\!\left(-1, \bigl\langle w, \phi(x_{j+1,j})\bigr\rangle\right) \nonumber \\
&\quad + \ell\!\left(1, \bigl\langle w, \phi(x_{p,q})\bigr\rangle\right)
      + \ell\!\left(-1, \bigl\langle w, \phi(x_{q,p})\bigr\rangle\right).
\label{S47}
\end{align}
where $\ell(y,\hat y) = (y - \hat y)^2$. The weights $w$ are then determined by minimizing the optimization problem
\begin{equation}
L(w) + \tfrac{1}{c}\|w\|_2^2,
\label{S48}
\end{equation}
where $c$ is the inverse strength of the regularization. The corresponding dual problem is given by
\begin{equation}
a := \bigl(K + \tfrac{1}{c} I \bigr)^{-1} y,
\label{S49}
\end{equation}
where $K=\phi(X)\phi(X)^T$ is the cross-sample similarity.
Writing $a$ as
\begin{equation}
a = \begin{pmatrix} b^{(1)} \\ b^{(2)} \\ h^{(1)} \\ h^{(2)} \end{pmatrix}, 
\quad b^{(1)} \in \mathbb{R}^{n-1}, \; b^{(2)} \in \mathbb{R}^{n-1}, h^{(1)} \in \mathbb{R}, h^{(2)} \in \mathbb{R}
\label{S51}
\end{equation}
we can write $(K + \tfrac{1}{c}I)^{-1} a = y$ as
\begin{align}
(\delta_s + \tfrac{1}{c}) b^{(1)}_i 
+ \delta_o\bigl(b^{(2)}_{i-1} + b^{(2)}_{i+1}\bigr)
+ \delta_o\bigl(h^{(1)}(\delta_{(i=p)} + \delta_{(i+1 = q)}) \bigr)
+ \delta_o\bigl(h^{(2)}(\delta_{(i=q)} + \delta_{(i+1 = p)}) \bigr) 
+ \kappa_d \langle a \rangle &= 1,
\nonumber\\
(\delta_s + \tfrac{1}{c}) b^{(2)}_i 
+ \delta_o\bigl(b^{(1)}_{i-1} + b^{(1)}_{i+1}\bigr)
+ \delta_o\bigl(h^{(2)}(\delta_{(i=p)} + \delta_{(i+1 = q)}) \bigr) 
+ \delta_o\bigl(h^{(1)}(\delta_{(i=q)} + \delta_{(i+1 = p)}) \bigr) 
+ \kappa_d \langle a \rangle &= -1, 
\nonumber\\
(\delta_s + \tfrac{1}{c}) h^{(1)} 
+ \delta_o\bigl(b^{(1)}_{q-1} + b^{(1)}_p\bigr) 
+ \delta_o\bigl(b^{(2)}_{q} + b^{(2)}_{p-1}\bigr) 
+ \kappa_d \langle a \rangle &= 1, 
\nonumber\\
(\delta_s + \tfrac{1}{c}) h^{(2)} 
+ \delta_o\bigl(b^{(2)}_{q-1} + b^{(2)}_p\bigr) 
+ \delta_o\bigl(b^{(1)}_{q} + b^{(1)}_{p-1}\bigr) 
+ \kappa_d \langle a \rangle &= -1, 
\nonumber\\
\label{S52}
\end{align}
where we define $\delta_s = \kappa_s - \kappa_d$, $\delta_o = \kappa_o - \kappa_d$, 
$\langle a \rangle := \sum_{i=1}^{2n} a_i$, and for notational simplicity, $b^{(j)}_0 = b^{(j)}_n = 0$. This means that if $b^{(1)},b^{(2)},h^{(1)},h^{(2)}$ are a solution to the equation, then so are
\begin{equation}
\tilde b^{(1)} := -b^{(2)}, 
\quad \tilde b^{(2)} := -b^{(1)},
\qquad \tilde h^{(2)} := -h^{(1)},
\quad \tilde h^{(2)} := -h^{(1)}.
\label{S53}
\end{equation}
Since (\ref{S49}) is well-defined, the solution $b$ must be unique and we can infer
\begin{equation}
b^{(1)} = -b^{(2)}, \qquad h^{(1)} = -h^{(2)}.
\label{S54}
\end{equation}
In particular, this means that
\begin{equation}
    \langle a\rangle=0.
\end{equation}

We thus define
\begin{equation}
\bar b := b^{(1)} = -b^{(2)}, 
\qquad \bar{h} := h^{(1)} = -h^{(2)},
\label{S55}
\end{equation}

Setting
\begin{equation}
\tilde c := c \delta_s,
\label{S56}
\end{equation}
(\ref{S52}) yields

\begin{align}
\delta_s(1 + \tilde{c}^{-1}) \bar{b}_i 
- \delta_o\bigl(\bar{b}_{i-1} + \bar{b}_{i+1}\bigr)
+ \delta_o\bigl(\bar{h}(\delta_{(i=p)} + \delta_{(i+1 = q)}) \bigr)
- \delta_o\bigl(\bar{h}(\delta_{(i=q)} + \delta_{(i+1 = p)}) \bigr) &= 1
\nonumber \\
\delta_s(1 + \tilde{c}^{-1}) \bar{h} 
+ \delta_o\bigl(\bar{b}_{q-1} + \bar{b}_p\bigr) 
- \delta_o\bigl(\bar{b}_{q} + \bar{b}_{p-1}\bigr) 
&= 1, 
\label{S57}
\end{align}
as $\langle a \rangle = \langle b^{(1)} \rangle + \langle b^{(2)} \rangle + h^{(1)} + h^{(2)} = 0$ and therefore the last summand vanishes. 

As $\tfrac{\delta_o}{\delta_s} = \tfrac{1-\alpha}{2}$, setting $\bar{b}$ in matrix form results in the following.  Note we denote $e_{j}$ as the $n$-dimensional one hot vector with a non-zero entry in the $j$-th dimension. We define
\begin{equation}
\tilde{K} \in \mathbb{R}^{(n-1)\times(n-1)},\qquad
\tilde{K}_{ij} =
\begin{cases}
1, & \text{if } i=j,\\[4pt]
-\dfrac{1-\alpha}{2}, & \text{if } |i-j|=1,\\[6pt]
0, & \text{else},
\end{cases}\quad B:=\tilde{K}+\tilde{c}^{-1}I.
\label{S60}
\end{equation}
Notably $(\delta_sB)^{-1}\mathbf{1}$ is the solution to the dual problem for TI without an exception. We therefore define $\bar{b}\TI:=(\delta_sB)^{-1}\mathbf{1}$. We then observe
\begin{align}& 
\delta_s(\tilde{K} + \tilde{c}^{-1} I)\bar{b} + \delta_o \bar{h} (e_p + e_{q-1} - e_{q} - e_{p-1} ) ] = 1 \nonumber \\
& \implies \delta_s 
B\bar{b} = 1 - \delta_o \bar{h} (e_p + e_{q-1} - e_{q} - e_{p-1}) )
\nonumber \\
&\implies
\bar{b} = (\delta_s B)^{-1} (1 - \delta_o \bar{h} (e_p + e_{q-1} - e_{q} - e_{p-1}) ) \nonumber \\ 
&\implies 
\bar{b} = \bar{b}\TI - \alpha' \bar{h}B^{-1} (e_p + e_{q-1} - e_{q} - e_{p-1}) ) \nonumber \\
& \implies \bar{b} = \bar{b}\TI - \alpha' \bar{h} (B^{-1}_{ip} + B^{-1}_{i(q-1)} - B^{-1}_{iq} - B^{-1}_{i(p-1)})
\label{S59}
\end{align}

The matrix $B$ is tridiagonal with diagonal entries \(1+\tilde{c}\) and
off–diagonal entries \(-\frac{1-\alpha}{2}\). 
The ratio between the diagonal and off–diagonal terms is therefore
\begin{equation}
-\,2\,\frac{1+\tilde{c}^{-1}}{1-\alpha}\;\le\;-2.
\label{S70}
\end{equation}

As this ratio is smaller than \(-2\), we can apply Theorem~\ref{thm:tri-inv} to infer
\begin{equation}
B^{-1}_{ij}
=
\frac{
\cosh\!\bigl((n-|j-i|)\,\lambda\bigr)
-
\cosh\!\bigl((n-i-j)\,\lambda\bigr)
}{
(1-\alpha)\,\sinh(\lambda)\,\sinh(n\lambda)
},
\qquad
\lambda := \operatorname{arccosh}\!\left(\frac{1+\tilde{c}^{-1}}{1-\alpha}\right).
\label{S71}
\end{equation}

For notational convenience we define $\alpha':=\delta_o/\delta_s=\tfrac{1-\alpha}{2}$. Thus the dual coefficients are:
\begin{align}
\bar{b} &= \bar{b}_i\TI - \alpha' \bar{h}(B\inv_{ip} + B\inv_{i(q-1)} - B\inv_{iq} - B\inv_{i(p-1)})) \nonumber \\
\bar{h} &= \tfrac{1}{ \delta_s (1 + \tilde{c}^{-1}) } \bigl( 1 - \delta_o (\bar{b}_p + \bar{b}_{q-1} - \bar{b}_q - \bar{b}_{p-1}) \bigr).
\label{S59}
\end{align}

The predictions of the fitted model can be expressed in terms of the dual coefficients by
\begin{align}
\begin{split}
f\!\left(x^{(j,k)}\right) 
&= \sum_{(x,y)\in\mathcal{D}} a_x K(x, x^{(j,k)})\\
&= \sum_{i=1}^{n-1} K(i,i+1; j,k)\, b_i^{(1)} + \sum_{i=1}^{n-1} K(i+1,i; j,k)\, b_i^{(2)} + K(p,q;j,k)h^{(1)} + K(q,p;j,k){h^{(2)}}.
\label{S61}
\end{split}
\end{align}

We can express this in terms of $\bar{b}$ and $\bar{h}$ as
\begin{equation}
f\!\left(x_{j,k}\right) 
= \sum_{i=1}^{n-1} K(i,i+1; j,k)\, \bar{b}_i - \sum_{i=1}^{n-1} K(i+1,i; j,k) \, \bar{b}_i + K(p,q;j,k)\bar{h} - K(q,p;j,k)\bar{h}.
\label{S62}
\end{equation}
There are now two different cases: either $x_{j,k}$ is in the training data or not.
\paragraph{Case 1: Test data.} 
If $x_{j,k}$ is not in the training data, all data points are either overlapping or distinct. 
In the first data point, a data point is overlapping if $i=j$ or $i=k-1$ and therefore
\begin{equation}
\sum_{i=1}^{n-1} K(i,i+1; j,k)\, \bar{b}_i + K(p,q;j,k)\bar{h}
= \kappa_d \langle \bar{b} \rangle + \delta_o(\bar{b}_j + \bar{b}_{k-1}) + \delta_o \bar{h}(\delta_{(j = p)} + \delta_{(k = q)})
\label{S63}
\end{equation}

(Note that one of these dual coefficients, as in the previous section, may not exist, if $i \in \{1,n\}$. 
To leave the expression simple, we define $\overline{b}_0 = \overline{b}_n = 0$, covering these cases as well.)
Analogously, the second sum amounts to
\begin{equation}
\sum_{i=1}^{n-1} K(i+1,i; j,k) \bar{b}_i + K(q,p;j,k)\bar{h}
= \kappa_d \langle \bar{b} \rangle + \delta_o(\bar{b}_{(j-1)} + \bar{b}_k) + \delta_o \bar{h}(\delta_{(j = q)} + \delta_{(k = p)})
\label{S64}
\end{equation}
and overall, the model behavior is expressed as
\begin{equation}
f(x_{j,k}) 
=  \delta_o \bigl(\bar{b}_j + \bar{b}_{k-1} - \bar{b}_{(j-1)} - \bar{b}_k\bigr) + \delta_o \bar{h} \bigl(\delta_{(j=p)} + \delta_{(k=q)} - \delta_{(j=q)} - \delta_{(k=p)}\bigr).
\label{S65}
\end{equation}

\bigskip

\textbf{Case 2: Training data.} \quad 
If $x_{j,k}$ is in the training data, then there are four cases.  Either $k=j+1$, $k=j-1$, or $(j,k)\in\{(p,q),(q,p)\}$.
In the first case, the first sum has as one of its data points $(j,k)=(j,j+1)$ itself. The only overlapping data points are potentially with the exception training pairs. 
\begin{equation}
\sum_{i=1}^{n-1} K(i,i+1; j,k)\, \bar{b}_i + K(p,q;j,j+1)\bar{h} 
=  \kappa_d \langle a \rangle + \delta_s(\bar{b}_j) + \delta_o \bar{h}(\delta_{(j=p)} + \delta_{(k=q)}) 
\label{S66}
\end{equation}
Analogously the second sum is:
\begin{equation}
\sum_{i=1}^{n-1} K(i+1,i; j,k) \bar{b} + K(q,p;j,j,k=j+1) \bar{h} 
= \kappa_d \langle \bar{a} \rangle + 
\delta_o(\bar{b}_{j-1} + \bar{b}_{k}) + 
\delta_o \bar{h}(\delta_{(j=q)} + \delta_{(k=p)}) 
\end{equation}
and overall, the model behavior is expressed as
\begin{equation}
f(x_{j,k}) 
= \delta_s\bar{b}_j - \delta_o(\bar{b}_{j-1} + \bar{b}_{k}) + \delta_o \bar{h}(\delta_{(j=p)} + \delta_{(k=q)} - \delta_{(j=q)} - \delta_{(k=p)}) 
\end{equation}
The second case is analogous: 
\begin{equation}
f(x_{j,k}) 
= - \delta_s\bar{b}_k + \delta_o(\bar{b}_{j} + \bar{b}_{k-1}) + \delta_o \bar{h}(\delta_{(j=p)} + \delta_{(k=q)} - \delta_{(j=q)} - \delta_{(k=p)}) 
\end{equation}
The third case is analogous to the first case.  Note that the similar data point is the exception training pair.
\begin{equation}
f(x_{j,k}) 
= \delta_s\bar{h} - \delta_o(\bar{b}_{p} + \bar{b}_{q-1} - \bar{b}_{p-1} + \bar{b}_{q})
\end{equation}
The last case is analogous to the fourth case.  
\begin{equation}
f(x_{j,k}) 
= - \delta_s\bar{h} + \delta_o(\bar{b}_{p} + \bar{b}_{q-1} - \bar{b}_{p-1} + \bar{b}_{q})
\end{equation}

Taken together, the prediction made by $\bar{a}$ are expressed as
\[
f(x_{j,k}) =
\begin{cases}
\delta_s\bar{b}_j - \delta_o(\bar{b}_{j-1} + \bar{b}_{k}) + \delta_o \bar{h}(\delta_{(j=p)} + \delta_{(k=q)} - \delta_{(j=q)} - \delta_{(k=p)}), & \text{if } k = j+1, \\[6pt]
- \delta_s\bar{b}_k + \delta_o(\bar{b}_{j} + \bar{b}_{k-1}) + \delta_o \bar{h}(\delta_{(j=p)} + \delta_{(k=q)} - \delta_{(j=q)} - \delta_{(k=p)}), & \text{if } k = j-1, \\[6pt]
\delta_s\bar{h} - \delta_o(\bar{b}_{p} + \bar{b}_{q-1} - \bar{b}_{p-1} - \bar{b}_{q}), & \text{if } j=p \text{ and } k=q, \\
- \delta_s\bar{h} + \delta_o(\bar{b}_{p} + \bar{b}_{q-1} - \bar{b}_{p-1} - \bar{b}_{q}), & \text{if } j=q \text{ and } k=p, \\
\delta_o \bigl(\bar{b}_j + \bar{b}_{k-1} - \bar{b}_{(j-1)} - \bar{b}_k\bigr) + \delta_o \bar{h} \bigl(\delta_{(j=p)} + \delta_{(k=q)} - \delta_{(j=q)} - \delta_{(k=p)}\bigr), & \text{else} 
\end{cases}
\]

This makes apparent the emergent rank representation for test cases:
\begin{equation}
r_j = \delta_o \bigl(\bar{b}_j - \bar{b}_{(j-1)} \bigr) + \delta_o \bar{h} \bigl(\delta_{(j=p)} - \delta_{(j=q)} \bigr),
\label{S67}
\end{equation}
i.e.\
\begin{equation}
    f(x_j,x_k)=r_j-r_k.
\end{equation}
\subsection{Reformulating $\bar{h}$}
Using the rank, can reformulate $\bar{h}$ as the following:
\begin{equation}
    \bar{h} = \frac{ 1 - {r}_p\TI + {r}_q\TI }
    { \delta_s(1+\tilde{c}^{-1}) - \delta_o(D_{pq} + D_{qp} - D_{pp} - D_{qq}) },
    \label{S68}
\end{equation}
where
\begin{align}
    D_{ij}&:= \delta_o \Big( B_{(i-1)j}^{-1} + B_{i(j-1)}^{-1} - B_{ij}^{-1} - B_{(i-1)(j-1)}^{-1} \Big),\\
    r\TI_j&:= \delta_o(\bar{b}_j\TI-\bar{b}_{j-1}\TI).
\end{align}

We note that we can use Lemma S1.4 \cite{lippl2024mathematical} to express
\begin{equation}
    r\TI_j=\frac{\sinh((\frac{n+1}{2}-j)\lambda)}{\sinh(\frac{n+1}{2}\lambda)-\sinh(\frac{n-1}{2}\lambda)}.
\end{equation}
We show this below.

We remind ourselves of the following.  

\begin{align}
\bar{b} &= {\bar{b}}_i\TI - \alpha' \bar{h}\bigl(B^{-1}_{ip} + B^{-1}_{i(q-1)} - B^{-1}_{iq} - B^{-1}_{i(p-1)}\bigr), \\
\bar{h} &= \tfrac{1}{ \delta_s (1 + \tilde{c}^{-1}) } \bigl( 1 - \delta_o (\bar{b}_p + \bar{b}_{q-1} - \bar{b}_q - \bar{b}_{p-1}) \bigr), \label{eq:h-expression}\\
r_j &= \delta_o \bigl(\bar{b}_j - \bar{b}_{(j-1)} \bigr) + \delta_o \bar{h} \bigl(\delta_{(j=p)} - \delta_{(j=q)} \bigr).
\label{A2}
\end{align}
Currently, $\bar{b}$ and $\bar{h}$ are interdependent---we will now solve for $\bar{h}$.

We define
\begin{equation*}
D_{ij} := \alpha'\bigl(B^{-1}_{(i-1)j} + B^{-1}_{i(j-1)} - B^{-1}_{ij} - B^{-1}_{(i-1)(j-1)}\bigr).
\label{D_def}
\end{equation*}

Then the first term of the rank can be written as
\begin{align}
\delta_o\bigl(\bar{b}_j - \bar{b}_{j-1}\bigr)
&= \underbrace{ \delta_o ({\bar{b}}_j\TI - {\bar{b}}_{j-1}\TI)}_{=\,r\TI_j} -  \alpha' \delta_o \bar{h}\Bigl[\bigl(B^{-1}_{jp} - B^{-1}_{(j-1)p}\bigr) + \bigl(B^{-1}_{j(q-1)} - B^{-1}_{(j-1)(q-1)}\bigr) \notag \\
&\hspace{6em} - \bigl(B^{-1}_{jq} - B^{-1}_{(j-1)q}\bigr) - \bigl(B^{-1}_{j(p-1)} - B^{-1}_{(j-1)(p-1)}\bigr)\Bigr] \\
&= r\TI_j + \delta_o\bar{h}\bigl(D_{jp} - D_{jq}\bigr).
\end{align}

Thus, the rank can be rewritten as
\begin{equation}
r_j = {r}_j^{\text{TI}} + \delta_o\bar{h}\bigl(D_{jp} - D_{jq}\bigr) + \delta_o\bar{h}\bigl(\delta_{j=p} - \delta_{j=q}\bigr).
\label{eq:rank-D}
\end{equation}

We now rewrite $\bar{h}$. Starting from (\ref{eq:h-expression})
and defining $Z := \delta_s(1+\tilde{c}^{-1})$,
\begin{equation}
Z\bar{h} = 1 - \delta_o\Bigl[(\bar{b}_p - \bar{b}_{p-1}) - (\bar{b}_q - \bar{b}_{q-1})\Bigr].
\end{equation}
Substituting $\delta_o(\bar{b}_j - \bar{b}_{j-1}) = {r}_j\TI + \delta_o\bar{h}(D_{jp} - D_{jq})$ at $j = p$ and $j = q$,

\begin{equation}
Z\bar{h} = 1 - \Bigl[\bigl(r\TI_p + \delta_o\bar{h}(D_{pp} - D_{pq})\bigr) - \bigl(r\TI_q + \delta_o\bar{h}(D_{qp} - D_{qq})\bigr)\Bigr].
\end{equation}
\begin{equation}
Z\bar{h} = 1 - r\TI_p + r\TI_q - \delta_o\bar{h}(D_{pp} - D_{pq}) + \delta_o\bar{h}(D_{qp} - D_{qq}).
\end{equation}
\begin{equation}
Z\bar{h} + \delta_o\bar{h}(D_{pp} - D_{pq} - D_{qp} + D_{qq}) = 1 - {r}_p^{\text{TI}} + {r}_q^{\text{TI}}.
\end{equation}
\begin{equation}
Z\bar{h} - \delta_o\bar{h}(-D_{pp} + D_{pq} + D_{qp} - D_{qq}) = 1 - {r}_p^{\text{TI}} + {r}_q^{\text{TI}}.
\end{equation}
Dividing by the factor on the left establishes (\ref{S68}).
\subsection{Characterizing $D_{ij}$}
We compute $D_{ij}$ in closed form. Recall
\begin{equation}
D_{ij} := \alpha'\bigl(B^{-1}_{(i-1)j} + B^{-1}_{i(j-1)} - B^{-1}_{ij} - B^{-1}_{(i-1)(j-1)}\bigr),
\end{equation}
and
\begin{equation}
\alpha' B^{-1}_{ij} = \alpha'\frac{\cosh\bigl((n-|j-i|)\lambda\bigr) - \cosh\bigl((n-i-j)\lambda\bigr)}{(1-\alpha)\sinh(\lambda)\sinh(n\lambda)} = \frac{\sinh\bigl((n-\max(i,j))\lambda\bigr)\sinh\bigl(\min(i,j)\,\lambda\bigr)}{\sinh(\lambda)\sinh(n\lambda)}.
\end{equation}
Applying the identity $\cosh A - \cosh B = 2\sinh\bigl(\tfrac{A+B}{2}\bigr)\sinh\bigl(\tfrac{A-B}{2}\bigr)$ with $A = (n-|j-i|)\lambda$ and $B = (n-i-j)\lambda$, and using $|j-i| + i + j = 2\max(i,j),i + j - |j-i| = 2\min(i,j)$, we find
\begin{equation}
\cosh\bigl((n-|j-i|)\lambda\bigr) - \cosh\bigl((n-i-j)\lambda\bigr) = 2\sinh\bigl((n-\max(i,j))\lambda\bigr)\sinh\bigl(\min(i,j)\,\lambda\bigr).
\end{equation}
Substituting and using $\alpha' = (1-\alpha)/2$, so that $2\alpha'/(1-\alpha) = 1$, gives the second equality. When $i \leq j$, this further simplifies to
\begin{equation}
\alpha' B^{-1}_{ij} = \frac{\sinh\bigl((n-j)\lambda\bigr)\sinh(i\lambda)}{\sinh(\lambda)\sinh(n\lambda)}.
\end{equation}

We split into three cases and use a hyperbolic trigonometric identities to simplify $D_ij$.
\paragraph{Case 1: $i \leq j-1$.}
\begin{align}
D_{ij} 
&= \frac{\bigl(\sinh(i\lambda) - \sinh((i-1)\lambda)\bigr)\bigl(\sinh((n+1-j)\lambda) - \sinh((n-j)\lambda)\bigr)}{\sinh(\lambda)\sinh(n\lambda)} \\
&= \frac{4\cosh\bigl((i-\tfrac{1}{2})\lambda\bigr)\sinh(\tfrac{\lambda}{2})\cosh\bigl((n-j+\tfrac{1}{2})\lambda\bigr)\sinh(\tfrac{\lambda}{2})}{\sinh(\lambda)\sinh(n\lambda)} \\
&= 2\cosh\bigl((i-\tfrac{1}{2})\lambda\bigr)\cosh\bigl((n-j+\tfrac{1}{2})\lambda\bigr)\frac{\tanh(\lambda/2)}{\sinh(n\lambda)}.
\end{align}

\paragraph{Case 2: $j \leq i-1$.}
By the symmetry $D_{ij} = D_{ji}$,
\begin{equation}
D_{ij} = 2\cosh\bigl((j-\tfrac{1}{2})\lambda\bigr)\cosh\bigl((n-i+\tfrac{1}{2})\lambda\bigr)\frac{\tanh(\lambda/2)}{\sinh(n\lambda)}.
\end{equation}

\paragraph{Case 3: $i = j$.}
By the symmetry of $B^{-1}$,
\begin{align}
D_{ii} &= \alpha'\bigl(2B^{-1}_{(i-1)i} - B^{-1}_{ii} - B^{-1}_{(i-1)(i-1)}\bigr) \\
&= \frac{2\sinh((n-i)\lambda)\sinh((i-1)\lambda) - \sinh((n-i)\lambda)\sinh(i\lambda) - \sinh((n-i+1)\lambda)\sinh((i-1)\lambda)}{\sinh(\lambda)\sinh(n\lambda)} \\
&= \frac{\sinh((i-1)\lambda)\bigl[\sinh((n-i)\lambda) - \sinh((n-i+1)\lambda)\bigr] + \sinh((n-i)\lambda)\bigl[\sinh((i-1)\lambda) - \sinh(i\lambda)\bigr]}{\sinh(\lambda)\sinh(n\lambda)} \\
&= \frac{-2\sinh(\tfrac{\lambda}{2})\Bigl[\sinh((i-1)\lambda)\cosh((n-i+\tfrac{1}{2})\lambda) + \sinh((n-i)\lambda)\cosh((i-\tfrac{1}{2})\lambda)\Bigr]}{\sinh(\lambda)\sinh(n\lambda)}.
\end{align}
Applying $2\sinh A\cosh B = \sinh(A+B) + \sinh(A-B)$,
\begin{align}
2\sinh((i-1)\lambda)\cosh((n-i+\tfrac{1}{2})\lambda) &= \sinh((n-\tfrac{1}{2})\lambda) + \sinh((2i-n-\tfrac{3}{2})\lambda), \\
2\sinh((n-i)\lambda)\cosh((i-\tfrac{1}{2})\lambda) &= \sinh((n-\tfrac{1}{2})\lambda) + \sinh((n-2i+\tfrac{1}{2})\lambda).
\end{align}
Adding and using $\sinh x + \sinh y = 2\sinh(\tfrac{x+y}{2})\cosh(\tfrac{x-y}{2})$ with $x+y = -\lambda$, $x-y = (4i-2n-2)\lambda$,
\begin{align}
\begin{split}
&\sinh((i-1)\lambda)\cosh((n-i+\tfrac{1}{2})\lambda) + \sinh((n-i)\lambda)\cosh((i-\tfrac{1}{2})\lambda) =\\ &\sinh((n-\tfrac{1}{2})\lambda) - \sinh(\tfrac{\lambda}{2})\cosh((2i-n-1)\lambda).
\end{split}
\end{align}
Therefore
\begin{align}
D_{ii} &= \frac{-2\sinh(\tfrac{\lambda}{2})\sinh((n-\tfrac{1}{2})\lambda) + 2\sinh^2(\tfrac{\lambda}{2})\cosh((2i-n-1)\lambda)}{\sinh(\lambda)\sinh(n\lambda)}.
\end{align}
Using $\sinh((n-\tfrac{1}{2})\lambda) = \sinh(n\lambda)\cosh(\tfrac{\lambda}{2}) - \cosh(n\lambda)\sinh(\tfrac{\lambda}{2})$, $\sinh(\lambda) = 2\sinh(\tfrac{\lambda}{2})\cosh(\tfrac{\lambda}{2})$, and $2\sinh^2(\tfrac{\lambda}{2}) = \cosh(\lambda) - 1$,
\begin{align}
-2\sinh(\tfrac{\lambda}{2})\sinh((n-\tfrac{1}{2})\lambda) &= -\sinh(\lambda)\sinh(n\lambda) + (\cosh(\lambda) - 1)\cosh(n\lambda), \\
2\sinh^2(\tfrac{\lambda}{2})\cosh((2i-n-1)\lambda) &= (\cosh(\lambda) - 1)\cosh((2i-n-1)\lambda).
\end{align}
Combining,
\begin{align}
D_{ii} &= \frac{-\sinh(\lambda)\sinh(n\lambda) + (\cosh(\lambda) - 1)\bigl[\cosh(n\lambda) + \cosh((2i-n-1)\lambda)\bigr]}{\sinh(\lambda)\sinh(n\lambda)}.
\end{align}
Applying $2\cosh A\cosh B = \cosh(A+B) + \cosh(A-B)$ with $A = (i-\tfrac{1}{2})\lambda$, $B = (n-i+\tfrac{1}{2})\lambda$,
\begin{equation}
\cosh(n\lambda) + \cosh((2i-n-1)\lambda) = 2\cosh\bigl((i-\tfrac{1}{2})\lambda\bigr)\cosh\bigl((n-i+\tfrac{1}{2})\lambda\bigr).
\end{equation}
Substituting and using $(\cosh(\lambda) - 1)/\sinh(\lambda) = \tanh(\lambda/2)$,
\begin{equation}
D_{ii} = -1 + 2\,\frac{\cosh\bigl((i-\tfrac{1}{2})\lambda\bigr)\cosh\bigl((n-i+\tfrac{1}{2})\lambda\bigr)\tanh(\lambda/2)}{\sinh(n\lambda)}.
\end{equation}

\paragraph{Combined form.}
Writing the off-diagonal cases uniformly in terms of $\min$ and $\max$, we find
\begin{equation}
\boxed{\,D_{ij} = -\delta_{i=j} + 2\,\dfrac{\cosh\bigl((i-\tfrac{1}{2})\lambda\bigr)\cosh\bigl((n-i+\tfrac{1}{2})\lambda\bigr)\tanh(\lambda/2)}{\sinh(n\lambda)}.}
\end{equation}
It will be useful to define a subpart of this:
\begin{equation}
    \tilde{D}_{ij}:=\cosh\bigl((\min(i,j)-\tfrac{1}{2})\lambda\bigr)\cosh\bigl((n-\max(i,j)+\tfrac{1}{2})\lambda\bigr),
    \label{eq:closed-form-D}
\end{equation}
Note that
\begin{equation}
D_{ij} = -\delta_{i=j} + 2\,\tilde{D}_{ij}\,\frac{\tanh(\lambda/2)}{\sinh(n\lambda)}.
    \label{eq:closed-form-D-2}
\end{equation}
We start from the generalized expression for the exception scaling factor, $\bar{h}$, which isolates the localized rank correction for the exception pair $(p, q)$. 
\begin{equation}
\bar{h} = \frac{1 - {r}_p^{\text{TI}} + {r}_q^{\text{TI}}}{Z - \delta_o(2D_{pq} - D_{pp} - D_{qq})}
\end{equation}
where the regularization factor is $Z := \delta_s(1+\tilde{c}^{-1})$. To evaluate the denominator, we expand the $D_{ij}$ terms. From \eqref{eq:closed-form-D-2}, we infer
\begin{equation}
\begin{split}
2D_{pq} - D_{pp} - D_{qq} =& 2\left(2\tilde{D}_{pq}\frac{\tanh(\lambda/2)}{\sinh(n\lambda)}\right) \\
&- \left(-1 + 2\tilde{D}_{pp}\frac{\tanh(\lambda/2)}{\sinh(n\lambda)}\right) - \left(-1 + 2\tilde{D}_{qq}\frac{\tanh(\lambda/2)}{\sinh(n\lambda)}\right) \\
= &2 + 4 \frac{\tanh(\lambda/2)}{\sinh(n\lambda)} \left( \tilde{D}_{pq} - \frac{1}{2}\tilde{D}_{pp} - \frac{1}{2}\tilde{D}_{qq} \right).
\end{split}
\end{equation}

We apply the hyperbolic difference identity $\cosh A - \cosh B = 2\sinh(\frac{A+B}{2})\sinh(\frac{A-B}{2})$ to the partial distance expressions within the parentheses. This allows us to factor out $\sinh\left(\frac{q-p}{2}\lambda\right)$, consolidating the remaining terms into
\begin{equation}
G := \cosh\left(\left(q - \frac{1}{2}\right)\lambda \right)\sinh\left(\left(n - \frac{p+q-1}{2}\right)\lambda \right) + \cosh\left(\left(n - p + \frac{1}{2}\right)\lambda \right)\sinh\left(\frac{p+q-1}{2}\lambda \right)
\end{equation}
Substituting $G$ back into our expanded term yields:
\begin{equation}
2D_{pq} - D_{pp} - D_{qq} = 2 + 4 \frac{\tanh(\lambda/2)}{\sinh(n\lambda)} \sinh\left(\frac{q-p}{2}\lambda\right) G.
\end{equation}
Finally, substituting this back into the denominator of our original equation transforms the regularized scaling factor $\bar{h}$ into a purely analytical form:
\begin{equation}
\bar{h} = \frac{1 - {r}_p^{\text{TI}} + {r}_q^{\text{TI}}}{(Z - 2\delta_o) - 4\delta_o \frac{\tanh(\lambda/2)}{\sinh(n\lambda)} \sinh\left(\frac{q-p}{2}\lambda\right) G}
\label{eq:h-expr}
\end{equation}
We then substitute $\bar{h}$ and the baseline transitive rank $\tilde{r}^{\text{TI}}$ back into the full rank expression~\eqref{eq:rank-D}:
\begin{equation}
r_i = \tilde{r}_i^{\text{TI}} + \delta_o \bar{h} (D_{ip} - D_{iq}) + \delta_o \bar{h} (\delta_{i=p} - \delta_{i=q})
\end{equation}
This is almost the final expression---we simply need to make a few further algebraic simplifications. First we note that
\begin{align}
\begin{split}
D_{ip} - D_{iq} &= \left(-\delta_{i=p} + 2\tilde{D}_{ip}\frac{\tanh(\lambda/2)}{\sinh(n\lambda)}\right) - \left(-\delta_{i=q} + 2\tilde{D}_{iq}\frac{\tanh(\lambda/2)}{\sinh(n\lambda)}\right)\\
&= -(\delta_{i=p} - \delta_{i=q}) + 2\frac{\tanh(\lambda/2)}{\sinh(n\lambda)}\left(\tilde{D}_{ip} - \tilde{D}_{iq}\right).
\end{split}
\end{align}
Hence,
\begin{align}
\begin{split}
r_i &= {r}_i^{\text{TI}} + \delta_o \bar{h} \left[ -(\delta_{i=p} - \delta_{i=q}) + 2\frac{\tanh(\lambda/2)}{\sinh(n\lambda)}\left(\tilde{D}_{ip} - \tilde{D}_{iq}\right) \right] + \delta_o \bar{h} (\delta_{i=p} - \delta_{i=q}) \\
&= {r}_i^{\text{TI}} + 2\delta_o \bar{h} \frac{\tanh(\lambda/2)}{\sinh(n\lambda)} \left(\tilde{D}_{ip} - \tilde{D}_{iq}\right).
\end{split}
\end{align}
Finally, plugging in \eqref{eq:h-expr} yields
\begin{align}
    \begin{split}
        r_i &= {r}_i^{\text{TI}} + \frac{ 2\delta_o \frac{\tanh(\lambda/2)}{\sinh(n\lambda)} \left(1 - \tilde{r}_p^{\text{TI}} + \tilde{r}_q^{\text{TI}}\right) \left(\tilde{D}_{ip} - \tilde{D}_{iq}\right) }{(Z - 2\delta_o) - 4\delta_o \frac{\tanh(\lambda/2)}{\sinh(n\lambda)} \sinh\left(\frac{q-p}{2}\lambda\right) G} \\
&= {r}_i^{\text{TI}} + \frac{\left(1 - {r}_p^{\text{TI}} + {r}_q^{\text{TI}}\right)\left(\tilde{D}_{ip} - \tilde{D}_{iq}\right)}{\frac{(Z - 2\delta_o)}{2\delta_o \frac{\tanh(\lambda/2)}{\sinh(n\lambda)}} - 2\sinh\left(\frac{q-p}{2}\lambda\right) G} \\
&= {r}_i^{\text{TI}} + \frac{\left(1 - {r}_p^{\text{TI}} + {r}_q^{\text{TI}}\right)\left(\tilde{D}_{ip} - \tilde{D}_{iq}\right)}{\sinh(\lambda)\sinh(n\lambda) - 2\sinh\left(\frac{q-p}{2}\lambda\right) G},
    \end{split}
\end{align}
where we used
\begin{equation}
    Z - 2\delta_o = \delta_s(\alpha + \tilde{c}^{-1}) = 2\delta_o\tanh(\tfrac{\lambda}{2})\sinh\lambda.
\end{equation}

\subsection{Expression of the training set behavior}

We derive the prediction on adjacent training pairs in terms of $r_j$, which already incorporates the exception correction.

Starting from the case-split expression for $k = j+1$,
\begin{equation}
f(x_j, x_{j+1}) = \delta_s \bar{b}_j - \delta_o(\bar{b}_{j-1} + \bar{b}_{j+1})
+ \delta_o \bar{h}\bigl(\delta_{j=p} + \delta_{j+1=q} - \delta_{j=q} - \delta_{j+1=p}\bigr),
\label{eq:fcase}
\end{equation}
we observe that the row-$j$ dual equation~\eqref{S57} reads
\begin{equation}
\delta_s(1 + \tilde{c}^{-1}) \bar{b}_j - \delta_o(\bar{b}_{j-1} + \bar{b}_{j+1})
+ \delta_o \bar{h}\bigl(\delta_{j=p} + \delta_{j+1=q} - \delta_{j=q} - \delta_{j+1=p}\bigr)-1 = 0.
\label{eq:dualrowj}
\end{equation}
Subtracting the left-hand side of~\eqref{eq:dualrowj} from~\eqref{eq:fcase}, the off-diagonal $\bar{b}$ terms
and the indicator terms cancel, yielding
\begin{equation}
f(x_j, x_{j+1}) = 1 - \delta_s \tilde{c}^{-1}\, \bar{b}_j.
\label{eq:fbbar}
\end{equation}

Next, we express $\bar{b}_j$ in terms of the full rank. From the rank
definition $r_j = \delta_o(\bar{b}_j - \bar{b}_{j-1}) + \delta_o \bar{h}(\delta_{j=p} - \delta_{j=q})$,
\begin{equation}
r_j - r_{j+1} = \delta_o(2\bar{b}_j - \bar{b}_{j-1} - \bar{b}_{j+1})
+ \delta_o \bar{h}\bigl(\delta_{j=p} - \delta_{j+1=p} - \delta_{j=q} + \delta_{j+1=q}\bigr).
\end{equation}
\eqref{eq:dualrowj} implies
\begin{equation}
     - \delta_o(\bar{b}_{j-1} + \bar{b}_{j+1})
+ \delta_o \bar{h}\bigl(\delta_{j=p} + \delta_{j+1=q} - \delta_{j=q} - \delta_{j+1=p}\bigr)=1-\delta_s(1 + \tilde{c}^{-1}) \bar{b}_j.
\end{equation}
This lets us infer
\begin{equation}
r_j - r_{j+1} = 2\delta_o \bar{b}_j - \delta_s(1 + \tilde{c}^{-1})\bar{b}_j + 1.
\end{equation}
Using $2\delta_o = \delta_s(1-\alpha)$,
\begin{equation}
r_j - r_{j+1} = 1 - \delta_s\bigl(\alpha + \tilde{c}^{-1}\bigr)\bar{b}_j
\quad\Longrightarrow\quad
\delta_s \bar{b}_j = \frac{1 - (r_j - r_{j+1})}{\alpha + \tilde{c}^{-1}}.
\label{eq:bbarinrank}
\end{equation}
Substituting~\eqref{eq:bbarinrank} into~\eqref{eq:fbbar},
\begin{equation}
f(x_j, x_{j+1}) = 1 - \frac{\tilde{c}^{-1}}{\alpha + \tilde{c}^{-1}}
\bigl[1 - (r_j - r_{j+1})\bigr]
= \frac{\alpha}{\alpha + \tilde{c}^{-1}} + \frac{\tilde{c}^{-1}}{\alpha + \tilde{c}^{-1}}(r_j - r_{j+1}).
\end{equation}
Defining $m := \alpha/(\alpha + \tilde{c}^{-1})$, so $1 - m = \tilde{c}^{-1}/(\alpha + \tilde{c}^{-1})$,
\begin{equation}
\,f(x_j, x_{j+1}) = m \cdot 1 + (1-m)(r_j - r_{j+1}),\qquad
m = \frac{\alpha}{\alpha + \tilde{c}^{-1}}.\,
\end{equation}
we have established the expression for the training case.
\subsection{Properties of the ranking system}
\label{app:ranking-system}
We find that within each transitive section $r_j$ is monotonically decreasing in $j$. We note that $\lambda>0$ and so $r\TI_j$ is monotonically decreasing in $j$. We recall
\begin{equation*}
    r_j^{\mathrm{pert}}(\alpha,\tilde{c})=\frac{\left(1 - {r}_p\TI + {r}_q\TI\right)\left(\tilde{D}_{jp} - \tilde{D}_{jq}\right)}{\sinh(\lambda)\sinh(n\lambda) - 2\sinh\left(\frac{q-p}{2}\lambda\right) G}.
\end{equation*}
Note that $1-r_p\TI+r_q\TI>0$. Further, because $\sinh\left(\frac{q-p}{2}\lambda\right)<0$ and $G>0$, the denominator is also positive. Hence monotonicity is a function of $\tilde{D}_{jp}-\tilde{D}_{jq}$. Recall
\begin{equation*}
    \tilde{D}_{ij}=\cosh\bigl((\min(i,j)-\tfrac{1}{2})\lambda\bigr)\cosh\bigl((n-\max(i,j)+\tfrac{1}{2})\lambda\bigr).
\end{equation*}
Hence, for $j\leq q$,
\begin{equation}
    \tilde{D}_{jp}-\tilde{D}_{jq}=\cosh((j-\tfrac12)\lambda)\left(\cosh\bigl((n-p+\tfrac{1}{2})\lambda\bigr)-\cosh\bigl((n-q+\tfrac{1}{2})\lambda\bigr)\right)
\end{equation}
Because $\cosh(x\lambda)$ is monotonically increasing in $x$,
\begin{equation}
    \cosh\bigl((n-p+\tfrac{1}{2})\lambda\bigr)-\cosh\bigl((n-q+\tfrac{1}{2})\lambda\bigr)<0.
\end{equation}
Because $\cosh((j-\tfrac12)\lambda)$ is monotonically increasing, the total expression is monotonically decreasing. The claim is established analogously for $j\geq p$.
\end{document}